\documentclass[journal]{IEEEtran}

\usepackage{amsfonts}
\usepackage{setspace}

\usepackage{setspace}
\usepackage{amssymb}
\usepackage{epsfig}
\usepackage{graphicx}
\usepackage{xcolor}
\usepackage{algorithm}
\usepackage{algorithmicx}
\usepackage{algpseudocode}
\usepackage{amsmath}
\usepackage{cite}
\usepackage{subcaption}

\floatname{algorithm}{Algorithm}

\usepackage[ps2pdf,bookmarks=false,
bookmarksnumbered=false, 
bookmarksopen=false,
colorlinks=false]{}

\usepackage[utf8]{inputenc}
\usepackage[english]{babel}
\usepackage{amsthm}
\theoremstyle{definition}
\newtheorem{definition}{Definition}
\newtheorem{theorem}{Theorem}

\newtheorem{lemma}[theorem]{Lemma}

\usepackage{tikz}
\newsavebox{\tempbox}

\usepackage{multicol}

\begin{document}

\title{Feature-based Federated Transfer Learning: Communication Efficiency, Robustness and Privacy}
\author{{Feng Wang,  M. Cenk Gursoy and Senem Velipasalar}
\thanks{The authors are with the Department of Electrical
		Engineering and Computer Science, Syracuse University, Syracuse, NY, 13244. E-mail: fwang26@syr.edu, mcgursoy@syr.edu, svelipas@syr.edu.}
\thanks{The material in this paper has been presented in part at the IEEE Global Communications Conference (Globecom), Dec. 2022.
}}
\maketitle

\thispagestyle{empty}

\begin{abstract}
In this paper, we propose feature-based federated transfer learning as a novel approach to improve communication efficiency by reducing the uplink payload by multiple orders of magnitude compared to that of existing approaches in federated learning and federated transfer learning. Specifically, in the proposed feature-based federated learning, we design the extracted features and outputs to be uploaded instead of parameter updates. For this distributed learning model, we determine the required payload and provide comparisons with the existing schemes. Subsequently, we analyze the robustness of feature-based federated transfer learning against packet loss, data insufficiency, and quantization. Finally, we address privacy considerations by defining and analyzing label privacy leakage and feature privacy leakage, and investigating mitigating approaches. For all aforementioned analyses, we evaluate the performance of the proposed learning scheme via experiments on an image classification task and a natural language processing task to demonstrate its effectiveness. \footnote{Code implementation available at: https://github.com/wfwf10/Feature-based-Federated-Transfer-Learning.}
\end{abstract}


\section{Introduction}

Federated learning (FL) is a form of distributed learning in which only model parameters are exchanged while datasets are kept local with the goal to maintain data privacy  \cite{konevcny2016federated}. Typically, 
in FL, a central server, known as the parameter server (PS), coordinates the collaborative training of a deep neural network (DNN) model \cite{kairouz2019advances} by aggregating updates on the weights and biases from multiple participating devices/clients. 
The widespread use of mobile phones and tablets with sufficient computational power and wireless communication capability enables FL in a wide range of applications such as speech recognition and image classification. Internet of Things (IoT) with large number of devices/sensors may generate even larger amount of data while casting further constraints on the computational power and transmission power \cite{mills2019communication}. With these, FL has gained widespread attention from both academic and industrial communities, resulting in a rapid increase in research on FL techniques, such as federated averaging (FedAvg) \cite{mcmahan2017communication}, federated transfer learning (FTL) \cite{cheng2022federated, yang2019federated}, and FL with differential privacy (DP) \cite{wei2020federated}. Among them, model-based FTL stands out as a particularly efficient scheme, because it transfers a well-trained source model into the target task of interest where samples may have different input and output spaces, and thus FTL requires fewer training samples and shortens the training process \cite{yin2019feature, 
chen2020fedhealth, ju2020federated, yang2020fedsteg}. 
Especially, with the  availability of open-source big data repositories and deep learning models \cite{pan2009survey, weiss2016survey}, transfer learning (TL) has become an attractive solution for various applications, such as text sentiment classification \cite{wang2011heterogeneous, kaya2013transfer, khan2019enhanced}, image classification \cite{duan2012learning, kulis2011you, zhu2011heterogeneous, shaha2018transfer, hussain2018study, han2018new}, human activity classification \cite{harel2010learning, park2016micro, agarwal2021transfer}, software defect classification \cite{ma2012transfer, nam2017heterogeneous}, and multi-language text classification \cite{prettenhofer2010cross, zhou2014heterogeneous, zhou2014hybrid}.

However, when the clients in FL are mobile devices and the training process is performed over a wireless network such as a wireless local area network (WLAN) or cellular system, it is crucial to minimize the data sent during updates of the model parameters. This is due to the limited availability of radio spectrum, and the uncertain nature of the wireless environment caused by factors such as channel fading and interference \cite{chen2021distributed}. In the context of FL applications on IoT devices, the added constraints on computation and power consumption for uploading data make it even more challenging to perform complex tasks that require deep neural networks (DNNs) with many layers and a large number of trainable parameters. While many existing FL studies have focused on shallow DNNs with a few layers, state-of-the-art DNN models used in various applications often have dozens of layers and millions or billions of trainable parameters in order to achieve the highest accuracy in e.g., image segmentation \cite{yuan2021segmentation, jain2021semask}, image classification \cite{dai2021coatnet, foret2020sharpness}, object detection \cite{liu2021swin, ghiasi2021simple}, question answering \cite{yamada2020luke}, medical image segmentation \cite{srivastava2021msrf}, and speech recognition \cite{zhang2020pushing}. 

\subsection{Contributions}
In response to the above-mentioned constraints on the uplink payload budget and limitations on the local computational power, we develop a novel scheme to perform model-based transfer FL and refer to this scheme as feature-based federated transfer learning (FbFTL). In FbFTL, rather than uploading the gradients, the input and output of the subset of DNN to be trained are uploaded with the goal to reduce the uplink payload. This is one of the key differences from prior FL and FTL strategies. In this paper, after describing the proposed FbFTL scheme, we analyze its overall uplink payload and provide comparisons with prior schemes to quantify the gains. Specifically, we test this approach by transferring the VGG-16 convolutional neural network (CNN) model \cite{simonyan2014very} trained with ImageNet dataset \cite{deng2009imagenet} to CIFAR-10 dataset \cite{krizhevsky2009learning}, and show that FL, two types of FTL, and FbFTL require uploading 3216 Tb, 949.5 Tb, 599 Tb, and 6.6 Gb of data, respectively, until performance convergence. With this, we demonstrate that the proposed FbFTL outperforms FL and FTL substantially by reducing the upload requirements by at least 5 orders of magnitude. We note that the ITU standard of 5G uplink user experienced data rate is only 50 Mb/s \cite{series2017minimum}, and even the 6G uplink user experienced data rate at Gbit/s level may not sufficiently support such huge uploading requirements of regular FL. Additionally, to our best knowledge, existing works on improving FL efficiency (such as FedAvg \cite{mcmahan2017communication}, sparsification \cite{sun2020toward, ozfatura2021time} and quantization \cite{he2021cossgd, konevcny2016federatedefficiency} with or without error feedback \cite{basu2019qsparse, richtarik2021ef21}, federated distillation \cite{li2019fedmd, lin2020ensemble, ahn2019wireless}, pruning \cite{jiang2022model, liu2020pruning} or partially trainable network \cite{sidahmed2021efficient, yang2022partial}, and over-the-air computation \cite{aygun2022hierarchical, yang2020federated, sun2022time}) still consider the transmission of gradient updates and can achieve a relatively limited reduction in payload and experience degradation in performance. For instance,  the payload reduction is of only two orders of magnitude of the original payload on the same CIFAR-10 dataset  \cite{shahid2021communication, xu2020compressed}. Therefore, FbFTL is still a more effective approach in reducing the uplink payload even after the efficiency of FL  has been improved via the aforementioned methods. We further show that FbFTL has substantially lower downlink payload, and requires significantly less computational power from the edge devices, and therefore it is much more friendly in terms of facilitating the training tasks on clients with limited power budget. 

To validate our approach, we further analyze the robustness of FbFTL against FL and FTL. We show significant reduction on packet loss rate (PLR) with FbFTL with the same block loss rate (BLR), and  demonstrate the robustness of FbFTL with insufficient training data. While there have been numerous studies on gradient quantization and sparsification to reduce the uplink payload \cite{reisizadeh2020fedpaq, shlezinger2020federated}, we illustrate that FbFTL also achieves significant compression rate by quantization while the validation accuracy is barely affected. Furthermore, we analyze the privacy leakage to a potential adversary. We define the label privacy leakage and feature privacy leakage for both feature-based and gradient-based frameworks. To our best knowledge, this is the first work that divides the privacy into different categories, defines each type, and proposes mitigation approaches. Via experimental results, we show that FbFTL has better performance (e.g., classification accuracy) with the same level of privacy protection when each client obtains a small set of samples.

In summary, our main contribution can be listed as follows: 
\begin{itemize}
    \item We propose the FbFTL scheme that uploads extracted features instead of gradients to reduce uplink payload by five orders of magnitude. 
    \item We analyze the robustness of FbFTL, and demonstrate that it maintains the payload advantage when strategies such as sparsification, quantization and error feedback are deployed.
    \item We study
    the  privacy guarantees in terms of entropy based formulations that quantify the uncertainty and also by utilizing differential privacy mechanisms.
    \item We show that a small batch size is preferred to protect label privacy of shuffled batches, and FbFTL has better performance (such as in terms of accuracy) given the differential privacy constraint for input privacy in a small batch.
\end{itemize}

\subsection{Organization and Notations}

\begin{table}
	\fontsize{8}{9}
	\selectfont
	\begin{center}
	    \caption{Notations}
		\begin{tabular}{  p{20mm} p{70mm}  }\hline
			\underline{FL Parameters:} &  Section \ref{subsec:FL} \\
			$\mathcal{U}$, $U$ & Set of clients, Number of clients \\ 
			$\mathcal{K}_u$ & Set of local training samples at client $u$  \\ 
			$K_u$ & Number of local training samples at client $u$ \\ 
			$\textbf{s}_{u,k}$ & The $k$th sample at client $u$ \\ 
			$\textbf{x}_{u,k}$ & Input vector of $\textbf{s}_{u,k}$ with $N_0$ attributes \\ 
			$\textbf{y}_{u,k}$ & Output vector of $\textbf{s}_{u,k}$ with $N$ attributes (or classes) \\ 
			$f_{\pmb{\theta}}(\textbf{x})$ & DNN that maps the input $\textbf{x}$ to estimated output $\hat{\textbf{y}}$ \\ 
			$\pmb{\theta}$ & Trainable parameter vector of DNN \\ 
			$L(f_{\pmb{\theta}} | \textbf{s}_{u,k})$ & Loss function \\ 
			$\textbf{g}_u$ & Sum of updates (or gradients) at client $u$ \\ 
			$\alpha$ & Learning rate \\ 
			$I$ & Number of communication iterations during training process   (superscript indicates methods) \\ 
			$C$ & Fraction of clients selected in each iteration \\ 
			\hline
			\underline{FTL Parameters:} & Section \ref{subsec:FTL}, \ref{subsec:FbFTLLearn}, \ref{subsec:FbFTLpayload} \\
			$M'$ & Number of layers without trainable parameters \\ 
			$M$ & Number of layers with trainable parameters \\ 
			$m_c$ & Index of the layer partitioning DNN into feature extraction part and task-specific part \\
			$f^1_{\pmb{\theta}^1}$ & Feature extraction sub-model with trainable parameters $\pmb{\theta}^1$ \\ 
			$f^2_{\pmb{\theta}^2}$ & Task-specific sub-model with trainable parameters $\pmb{\theta}^2$ \\ 
			$\textbf{z}_{u,k}$ & Extracted features $f^1_{\pmb{\theta}^1}(\textbf{x}_{u,k})$ \\
			$T_m$ & Number of trainable parameters in the $m$th trainable layer \\ 
			$N^{-}_{m}$ & Number of input nodes at the $m$th trainable layer  \\ 
			$N^{+}_{m}$ & Number of output nodes at the $m$th trainable layer \\ 
			\hline
			\underline{Performance Measurements:} & \hspace{30pt} Section \ref{subsec:FL}, \ref{subsec:FbFTLpayload}, ,\ref{subsec:FbFTLcomplex} \\
			$d$ & Payload of each float number in bits \\ 
			$P$ & Uplink payload  (superscript indicates methods) \\ 
			$D$ & Downlink payload  (superscript indicates methods) \\ 
			$O(X)$ & Computation time complexity for training \\
			\hline
			\underline{Robustness Measurements:} & \hspace{25pt} Section \ref{sec:robust} \\
			PLR & Packet loss rate \\ 
			BLR & Block loss rate \\ 
			$n_b$ & Number of blocks in each packet \\ 
			$n_r$ & Maximum number of packet retransmission allowed \\ 
            $\mathcal{S}_r(\cdot)$ & Sparsification function keeping a ratio $r$ of the input elements \\
            $\mathcal{Q}_q(\cdot)$ & Quantization function keeping $q$ bits of the input elements \\
            $\textbf{m}_{u,t}$ & Local memory vector for error feedback \\
			\hline
			\underline{Privacy Measurements:} & \hspace{10pt} Section \ref{sec:privacy} \\
			$H( \cdot )$ & Entropy to quantify the  uncertainty \\ 
			$N$ & Output dimension (or number of labels) \\ 
			$\mathcal{B}$ & Set of possible batches \\ 
			$|\mathcal{B}|$ & Number of possible batches \\ 
			$B$ & Unknown batch type of client $u$ as a random variable \\ 
			$P_y$ & Distribution over labels  \\ 
			$P_{y,\text{uni}}$ & Uniform label distribution  \\ 
			$n^y_{a,b}$ & Count of output label type $a$ in a batch of type $b$ \\ 
			$\textbf{P}_{B|\text{uni}}$ & Adversary's presumed probability distribution of batches without prior knowledge \\ 
			$H(B\ |\ P_{y,\text{uni}} )$ & Adversary's uncertainty of a batch without prior knowledge of the content \\ 
            $P_{y,\text{true}}$ & True label distribution  \\ 
			$\textbf{P}_{B|\text{true}}$ & Adversary's presumed probability distribution of batches given the true sample distribution \\ 
			$H\{ B\ |\ P_{y,\text{true}}\}$ & Adversary's uncertainty of a batch given the true sample distribution $P_{y,\text{true}}$ \\
			$L^l_s$ & Label Privacy Leakage with Statistical Information \\ 
			$c(b)$ & Count for batches of different types among shuffled batches \\ 
			$\mathcal{C_U}$ & Number of each batch type $b$ in shuffled batches as a set of random variables \\ 
			$\textbf{P}_{B,\text{emp}}$ & Adversary's presumed empirical distribution of batches given shuffled batches $\mathcal{C_U} = c_U = [c(1), c(2), \ldots, c(|\mathcal{B}|)]$ \\ 
			$H(B | \mathcal{C_U} = c_U)$ & Adversary's uncertainty of a batch given shuffled batches $\mathcal{C_U}$ \\
			$L^l_q$ & Label Privacy Leakage with Query \\ 
            $L_t^l$ & Total Label Privacy Leakage \\
			$\textbf{n}$ & Additive independent Gaussian noise to protect local instance privacy, where $\textbf{n} \sim \mathcal{N}(0,\,\sigma^2 R^2 \text{std}^2 (  g(\textbf{d}_1) ) \textbf{I})$ \\ 
			$\textbf{d}$ & Training samples for one batch \\ 
			$g(\textbf{d})$ & Uploading vector in one batch \\ 
			$(\epsilon, \delta)$ & Differential privacy level \\ 
			$R$ & Maximum relative distance \\  \hline
            Additionally: & $f$ in subscript ($\star_f$) denotes transfer learning updating full model, and $c$ in subscript ($\star_c$) denotes transfer learning updating task-specific sub-model \\ \hline
		\end{tabular}
		\label{tab:notation}
	\end{center}
\end{table}

The remainder of the paper is organized as follows. In Section \ref{related}, we discuss the preliminary concepts regarding FL and FTL, describe the system design via FedAvg, and demonstrate their uplink payload. In Section \ref{sec:FbFTL}, we introduce the proposed FbFTL algorithm, introduce the learning structure and system design, and compare its payload with FL and FTL. In this section, we also evaluate the performance of FbFTL via simulations, and provide comparisons with FL and FTL. In Section \ref{sec:robust}, we illustrate the robustness of FbFTL in the presence of packet loss, data insufficiency, and quantization. In Section \ref{sec:privacy}, we define the label privacy leakage and the feature privacy leakage. For label privacy, we discuss the leakage with statistical information and the leakage with query, and investigate mitigation approaches to avoid these privacy leakages. For feature privacy leakage, we illustrate the mitigation approach via differential privacy. Finally, we conclude the paper in Section \ref{sec:con}. The list of all notations in the paper is provided in Table \ref{tab:notation} as a quick reference.

\section{Preliminaries} \label{related}  
In this section, we introduce preliminary concepts related to FL and FTL, and analyze the requirements on their successful uplink payload.

\subsection{Federated Learning} \label{subsec:FL}
We address a common FL task in which a PS coordinates a set $\mathcal{U}$ of $U$ clients to cooperatively train a DNN model. Each client $u$ possesses a local dataset $\mathcal{K}_u$ with $K_u$ samples for training. In this dataset,  the $k$th sample $\textbf{s}_{u,k} \in \mathcal{K}_u$ is comprised of an input vector $\textbf{x}_{u,k} \in \mathbb{R}^{N_0}$ and an output vector $\textbf{y}_{u,k} \in \mathbb{R}^N$. The DNN, represented by the function $f_{\pmb{\theta}}(\textbf{x})$, maps the input $\textbf{x}$ to an  output $\hat{\textbf{y}}$ (as an estimate of $\textbf{y}$) with trainable parameter vector $\pmb{\theta}$. The goal of the FL training process  is to update the parameter vector $\pmb{\theta}$ using the training samples from each client in order to minimize the expected loss $\mathbb{E} (L(f_{\pmb{\theta}} | \textbf{s}_{test}) )$ on the unseen sample $\textbf{s}_{test}$ with the same distribution as the training samples. For most DNNs with classification tasks, the output label $\textbf{y}$ is an axis-aligned unit vector with one element equal to 1 indicating the class of this sample, and all others equal to 0. In this case, the loss function is typically the categorical cross-entropy:
\begin{equation} \label{eq:crossentropy}
L(f_{\pmb{\theta}} | \textbf{s}_{u,k}) = - \sum_{n=1}^{N} {\pmb{y}_{u,k}[n] \log{f_{\pmb{\theta}}(\pmb{x}_{u,k})[n]}} .
\end{equation}

In order to minimize the loss and keep the training samples local, the authors in \cite{mcmahan2017communication} proposed an iterative distributed optimization scheme called FedAvg with the following steps:
\begin{enumerate}
    \item The PS initiates trainable parameters $\pmb{\theta}$, and broadcasts the model structure $f$ with non-trainable parameters to each client.
    \item The PS chooses a subset of $UC$ clients (with $C$ denoting the fraction of clients in this iteration) and broadcasts the trainable parameters $\pmb{\theta}$.
    \item Each client performs stochastic gradient descent (SGD) to obtain the parameter update corresponding to each training sample: $\nabla_{\pmb{\theta}}{L(f_{\pmb{\theta}} | \textbf{s}_{u,k})}$.
    \item Each client sends the sum of the updates over all local samples $\textbf{g}_u = \sum_{k=1}^{K_u}{\nabla_{\pmb{\theta}}{L(f_{\pmb{\theta}} | \textbf{s}_{u,k})}}$ and $K_u$ to the PS.
    \item The PS updates the parameter $\pmb{\theta}$  with $\pmb{\theta} - \sum_{u=1}^{U}{\frac{\alpha}{K_u} \textbf{g}_u}$ where $\alpha$ is the learning rate that controls the training speed.
    \item Return to step 2) until convergence.
\end{enumerate}

Let us assume that the number of iterations in FL with FedAvg is $I^{FL}$ (in the absence of any  processing and transmission failures). This implies that the total number of times the clients upload $\textbf{g}_u$ is $I^{FL} U C$ within the training process. Assume further that the DNN includes $M$ layers with trainable parameters (such as convolutional layer and fully connected layer, i.e., dense layer) and $M'$ layers without trainable parameters (such as pooling layer and residue connection).  The $m$th trainable layer has $T_m$ trainable parameters. Thus, for each trainable parameter, the client uploads one float number of $d$ bits for the corresponding element in the update $\textbf{g}_u$ during each iteration. Therefore, the overall uplink payload to train a DNN via FL with FedAvg is 
\begin{equation} \label{eq:PayloadFL}
P^{FL} = d I^{FL} U C \sum_{m=1}^{M}{T_m} \  \text{bits},
\end{equation}
and we will set this as a benchmark to compare with three other training methods in the remainder of this paper.

\subsection{Federated Transfer Learning} \label{subsec:FTL}
TL is an effective learning technique that utilizes knowledge from a different domain to enhance the performance in the target domain. FTL incorporates the privacy-preserving distributed learning paradigm into conventional TL in order to address the challenges of spectrum limitations in wireless applications and training data scarcity. In this paper, we consider FTL to be performed in the same fashion as in the previous subsection, where one PS orchestrates $U$ clients. Based on the difference between the domains, FTL can be divided into three different categories: instance-based FTL, feature-based FTL, and model-based FTL \cite{yang2019federated, 5640675}. The former two types of FTL assume similarity in the distribution of the input and output, and hence we in this paper focus on the model-based FTL which only assumes similarity in the functionality to extract a high-dimensional description from the input data. 

\begin{figure}[h]
  \centering
  \includegraphics[width=0.5\textwidth]{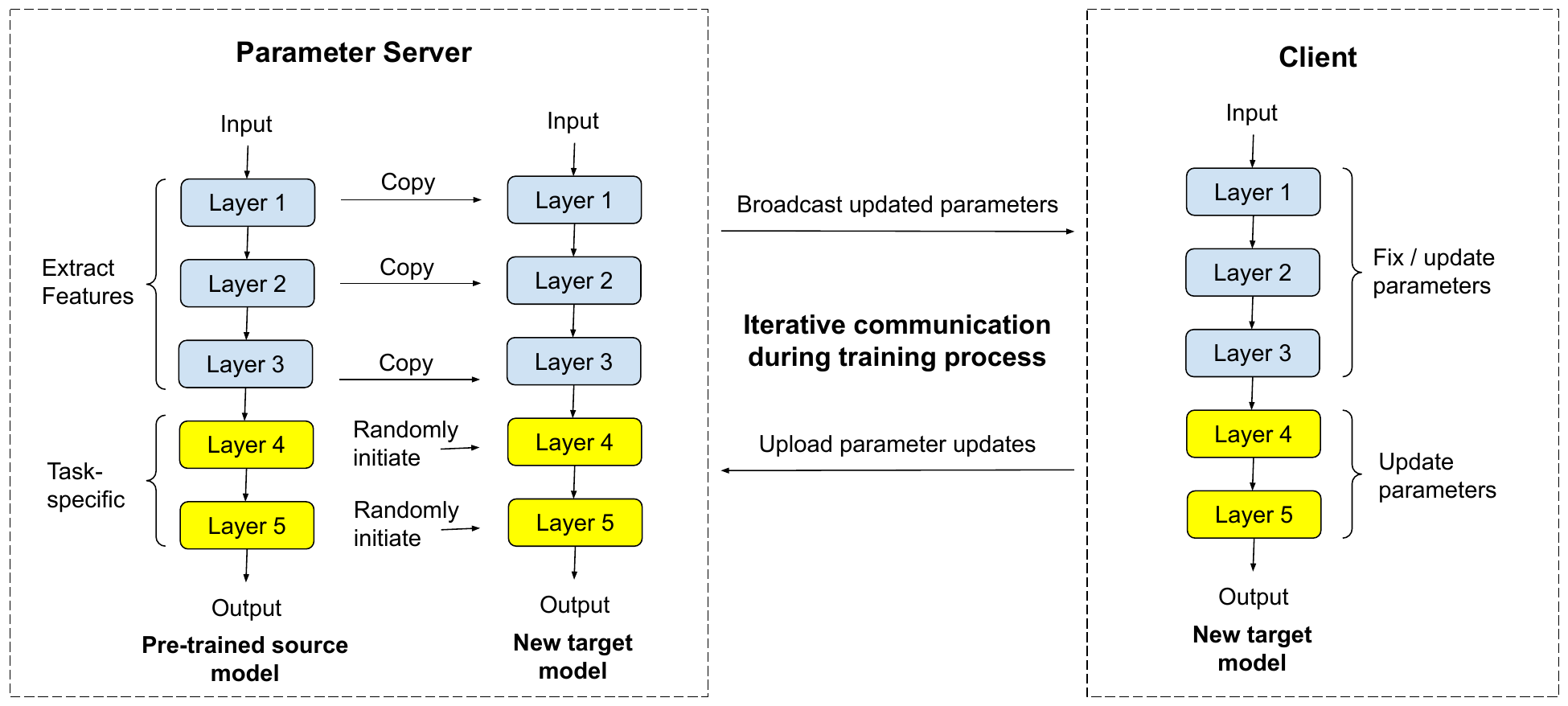}
  \caption{Diagram of the iterative training process of model-based federated transfer learning.}
  \label{fig:FTLdiagram}
\end{figure}
As shown in Fig. \ref{fig:FTLdiagram}, we consider FTL with a DNN model pre-trained with an open-source dataset that corresponds to a similar but not the same task as the source model. The goal here is to reduce the number of iterations during training. Many DNNs can be partitioned into two sections. The first section generates the high-dimensional features from the sample input data with general information. The second section carries out operations for a particular task, such as classification, and is less likely to be transferable to a new task. To move the knowledge of the source model into a new task before training, the feature extraction section of the source model is directly transferred to the new target model, while the task-specific part is randomly initiated. FTL can be performed by retraining all the parameters with new data. However, there typically exist a large  number of parameters to train, and in order to reduce the training time and data requirements, a common approach is to perform FTL by fixing the feature extraction part and simply updating the task-specific part. In particular, we select one fully connected layer $m_c$ (which is Layer 4 in Fig. \ref{fig:FTLdiagram}), close to the output without a paralleling path (such as residue connection), and randomly initiate all trainable parameters of layers $m_c, m_c+1, \ldots, M$ (which are Layer 4 and Layer 5 in Fig. \ref{fig:FTLdiagram}). During the distributed training process, all copied parameters are fixed, and we only update the parameters of layers $m_c, \ldots, M$.

Similar to the previous sub-section on FL, we determine the least requirements on successful uplink payload for FTL until convergence. For FTL that updates all parameters, assuming that FTL with FedAvg is iterated $I^{FTL}_{f}$ times, the overall uplink payload required for training is 
\begin{equation} \label{eq:PayloadFTLfull}
P^{FTL}_{f} = d I^{FTL}_{f} U C \sum_{m=1}^{M}{T_m} \  \text{bits}.
\end{equation}
On the other hand, assuming FTL with FedAvg that only updates the task-specific part with $I^{FTL}_{c}$ iterations and assuming that each sum update $\textbf{g}_u$ consists of $\sum_{m=m_c}^{M}{T_m}$ trainable parameters, the overall uplink payload for training is 
\begin{equation} \label{eq:PayloadFTLclassify}
P^{FTL}_{c} = d I^{FTL}_{c} U C \sum_{m=m_c}^{M}{T_m} \  \text{bits}.
\end{equation}
Typically, training from scratch requires more training samples, and thus we have $I^{FTL}_{c}UC \approx I^{FTL}_{f}UC < I^{FL}UC$. For most of the models, the majority of the DNN structure is dedicated to the feature extraction, and obviously $\sum_{m=m_c}^{M}{T_m} < \sum_{m=1}^{M}{T_m}$. Therefore, 
\begin{equation} \label{eq:compareFLFTL}
P^{FTL}_{c} < P^{FTL}_{f} < P^{FL}.
\end{equation}

\subsection{Other Recent Transfer and Federated Learning Schemes} \label{subsec:others}

\begin{table*}[t]
    \small
	\begin{center}
		\begin{tabular}{ | l || l | l | l | l | } \hline
		Method & Payload & Computation 
  & Key Assumption  \\ \hline\hline
		FL & iterative gradient-level & medium, distributed 
  & start with random initialization  \\ \hline
		FTL & iterative gradient-level & low, distributed 
  & start with source model  \\ \hline
		\textbf{FbFTL} (ours) & \textbf{one-time feature-level} & low, mostly centralized 
  & start with source model  \\ \hline
	    EWC \cite{kirkpatrick2017overcoming} & iterative gradient-level & high, distributed 
     & evaluating many different tasks  \\ \hline
		SFL \cite{thapa2022splitfed} & iterative gradient-level & medium, hybrid of distributed and  centralized 
  & two servers needed \\ \hline
		PFL \cite{fallah2020personalized} & iterative gradient-level & very high, distributed 
  & large and heterogeneous local batch  \\ \hline
		\end{tabular}
    	\caption{Comparison between different federated learning methods}
    	\label{tab:compare}
	\end{center}
\end{table*}

While we in this paper focus on FL schemes that generalize to the majority of use cases, there are also recent works that improve the performance under certain assumptions and can be used in FL. We list these well-known methods, along with FL, FTL, and the proposed FbFTL, in Table \ref{tab:compare} and provide a comparison. Below, we briefly describe the key assumptions of these methods, and highlight the differences of the proposed FbFTL scheme. 

Specifically, elastic weight consolidation (EWC) \cite{kirkpatrick2017overcoming} focuses on a large number of different tasks simultaneously and the transfer between them by remembering the importance of each weight. The performance is typically measured on all experienced tasks. EWC is inspired by neuroscience and Bayesian inference, and it aims to quantify the importance of each weight to the tasks the model has previously learned. The fundamental idea is that the weights critical to prior tasks should be altered less when new data is encountered. In this paper, we consider a different problem, where we transfer to a single target task from a single source model on another task whose performance is no longer considered. In this problem setting, there is no clear difference between EWC and plain SGD in  \cite{kirkpatrick2017overcoming} in terms of both training performance or training time, but EWC does introduce extra communication payload and computation cost in FL. On the contrary, FbFTL achieves five orders of magnitude payload reduction.

Split federated learning (SFL) \cite{thapa2022splitfed} partitions the neural network into two segments: the initial part near the input undergoes training in a federated learning manner using a ``Fed Server", and the subsequent part near the output is trained using split learning with a ``Main Server," which receives the smashed data. Compared to FbFTL, the federated learning segment utilizing the Fed Server operates in the same way as FL and requires significant uplink and downlink payload. Hence, this part alone is typically much more communication-expensive compared to FbFTL. The advantage of SFL lies in its similar performance to FL and not needing a pre-trained model as it updates all parameters. We will present the experimental performance comparison in Table \ref{tab:exp} in Section \ref{sec:exp}. 

Personalized federated learning (PFL) \cite{fallah2020personalized} considers clients with large number of heterogeneous local samples, enabling generalization to the heterogeneous local distribution. However, we in this paper consider a large number of clients with a small set of  local samples from each. As explained in Section \ref{subsec:label_pri} and Section \ref{subsec:exp_pri}, such a setting better protects label privacy. Therefore, personalized federated learning focuses on a different problem from ours. We cannot apply this approach to FbFTL because personalized FbFTL violates the shuffle-batch assumption, unless we train the global model by FbFTL and then retrain locally. Although PFL has better local performance than FL with large heterogeneous batches, it requires several independent local batches and significantly more local computation. Furthermore, it does not allow quantization and sparsification since the weights instead of gradients are uploaded. Therefore, it is comparatively even less communication-efficient when quantization and sparsification are deployed in other methods, as shown in Section \ref{subsec:quantization} and Table \ref{tab:compress}.

\section{Feature-based Federated Transfer Learning}\label{sec:FbFTL}
In this section, we describe the proposed FbFTL framework and demonstrate that it requires a substantially smaller uplink payload compared to FL and FTL. This efficiency arises from the fact that  the extracted features and outputs are uploaded rather than the parameter updates.

\subsection{Learning Structure} \label{subsec:FbFTLLearn}
\begin{figure}[h]
  \centering
  \includegraphics[width=0.5\textwidth]{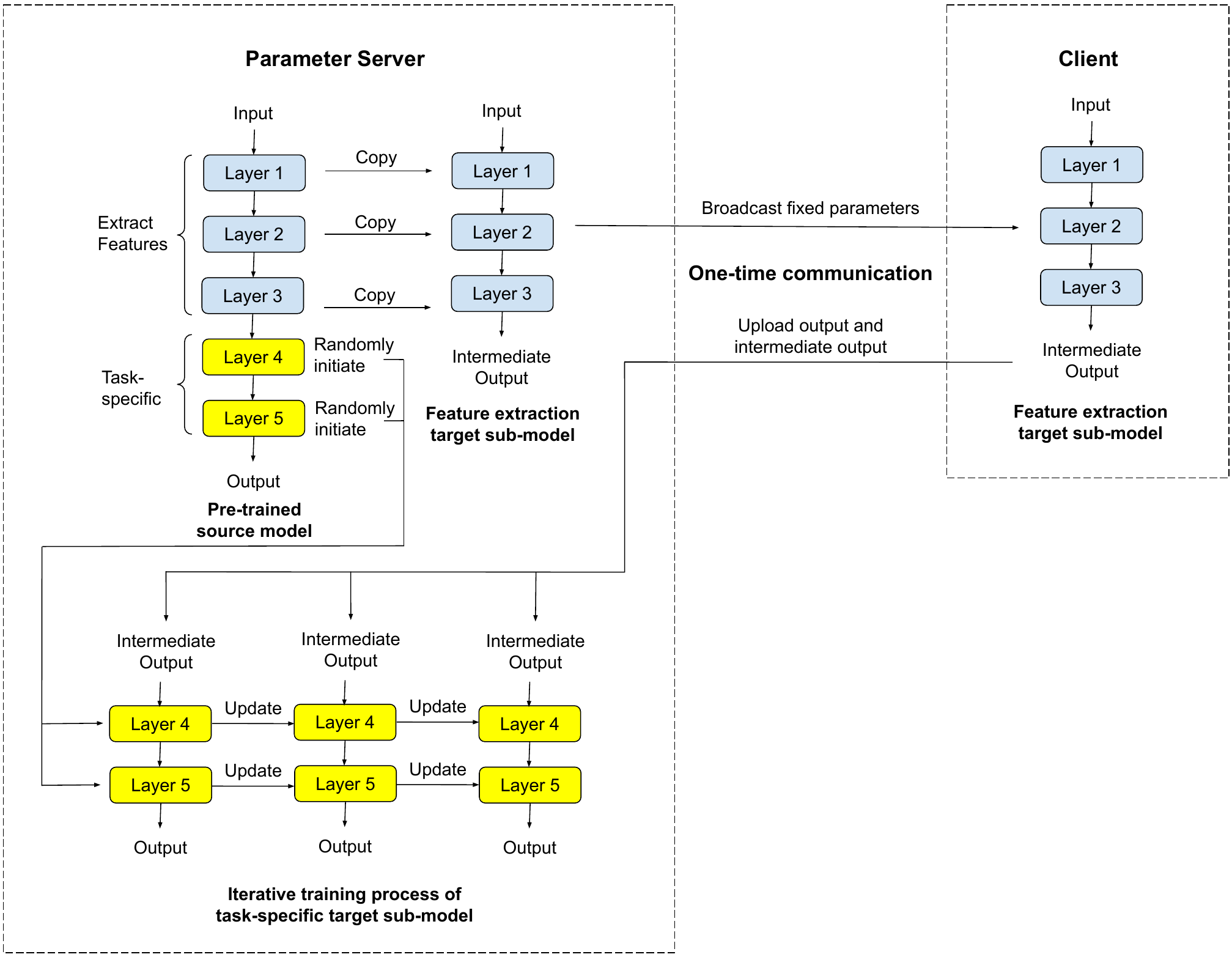}
  \caption{Diagram of the training process of feature-based federated transfer learning, which reduces to one-time communication of the output and the intermediate output (i.e., extracted features).}
  \label{fig:FbFTLdiagram}
\end{figure}
As depicted in Fig. \ref{fig:FbFTLdiagram}, in FbFTL we consider the model-based TL on a source DNN model $f'$ pre-trained on a different task. We choose one fully connected layer $m_c$ without paralleling path as the cut layer, and divide the new target model $f \leftarrow f'$ into two parts. Those layers before layer $m_c$ (which is layer 4 in Fig. \ref{fig:FbFTLdiagram}, i.e., $m_c = 4$ ) are regarded as the feature extraction sub-model $f^1_{\pmb{\theta}^1}$, and the parameters $\pmb{\theta}^1$ for the new target model $f$ are copied from the pre-trained source model $f'$ and fixed without any further update. The other layers are regarded as the task-specific sub-model $f^2_{\pmb{\theta}^2}$, and all trainable parameters $\pmb{\theta}^2$ in the new target model $f$ are randomly initiated for training with the $u$th client's $k$th training sample $\textbf{s}_{u,k} = \{ \textbf{x}_{u,k}, \textbf{y}_{u,k}\}$ for all $u$ and $k$. 

The forward pass of the full model $f_{\pmb{\theta}}(\textbf{x}_{u,k}) = f^2_{\pmb{\theta}^2}(f^1_{\pmb{\theta}^1}(\textbf{x}_{u,k}))$ maps the input $\textbf{x}_{u,k}$ to the estimated output $\hat{\textbf{y}}_{u,k}$, and we note that the output of the feature extraction sub-model $\textbf{z}_{u,k} = f^1_{\pmb{\theta}^1}(\textbf{x}_{u,k})$ is also the input of the task-specific sub-model $f^2_{\pmb{\theta}^2}$. Note that we require only the task-specific sub-model to be trained. In this setting, each client generates the features $\textbf{z}_{u,k}$ from input $\textbf{x}_{u,k}$ and sends these features to the PS only once. Subsequently, in FbFTL, the PS performs the gradient back-propagation iteratively without sending any feedback to the clients. Contrary to this, in FL and FTL, the parameter update, which is based on the gradients, highly relies on the parameters in the current training iteration, and the same data sample may generate different parameter updates in different iterations within the training process. Therefore, in FL and FTL, we either require many more training samples, or need to upload updates multiple times for the same sample. However in FbFTL, each client only needs to upload the intermediate output $\textbf{z}_{u,k}$ and output $\textbf{y}_{u,k}$ once, instead of iteratively uploading the gradients. The PS may deem these as the input and the output of the training sample for the task-specific sub-model $f^2_{\pmb{\theta}^2}$. Such pairs of samples are not correlated with the model parameters $\pmb{\theta}^2$, and therefore they can be used in different training iterations without downloading or uploading anything again. We provide the steps of the FbFTL algorithm in Algorithm \ref{alg:FbFTL} below.

\begin{algorithm}
\caption{FbFTL}
	\label{alg:FbFTL}
	\begin{algorithmic}
	    \State{PS copies fixed sub-model $f^1_{\pmb{\theta}^1}$ from pre-trained source model and randomly initiates trainable sub-model $f^2_{\pmb{\theta}^2}$}.
	    \State{PS broadcasts $f^1_{\pmb{\theta}^1}$ to each client $u \in \mathcal{U}$}
	    \State{\textbf{Clients execute:}}
	    \For{client $u \in \mathcal{U}$}
	        \For{sample $\textbf{s}_{u,k} \in \mathcal{K}_u$}
	            \State {Upload $\textbf{z}_{u,k} = f^1_{\pmb{\theta}^1}(\textbf{x}_{u,k})$ and $\textbf{y}_{u,k}$ to PS}
	        \EndFor
	    \EndFor
	    \State{\textbf{PS executes:}}
	    \While{$not\ converge$}{
	        \For{mini-batch $b$ of pairs $\{u, k\} \in \mathcal{U} \times \mathcal{K}_u$}
	            \State{$\pmb{\theta}^2 \leftarrow \pmb{\theta}^2 - \alpha \sum_{\{u, k\} \in b}{\nabla_{\pmb{\theta}^2}{L(f_{\pmb{\theta}^2} | \textbf{z}_{u,k}, \textbf{y}_{u,k})}}$}
	        \EndFor
	    \EndWhile}
	\end{algorithmic}
\end{algorithm}

As emphasized above, a key distinction of FbFTL is that the training samples of task-specific sub-model are uploaded rather than the direct gradient updates. With this, FbFTL provides additional important benefits that can lead to further improvements in the training performance and efficient management in practice. One of the most important benefits is the significant reduction of packet loss rate given the same batch loss rate since FbFTL has much less data in each batch, as we will illustrate in detail in Section \ref{subsec:packet_loss}. One other benefit is the waiving of the requirement of synchronization between clients, since there is only one iteration in FbFTL. The other benefits pertain to hyper-parameter fine-tuning, dataset balancing, and enabling flexible training batch size selection, as detailed below. 

Specifically, one of the most important additional benefit is that FbFTL enables iterative fine-tuning of the SGD optimizer hyper-parameters such as the learning rate. To obtain the optimized DNN, one needs to find the optimized hyper-parameters that provide the best convergence performance. In FbFTL, the model can be trained from scratch several times at the PS to identify the best hyper-parameter setting without additional communication with the clients. On the other hand, in FL and FTL, the entire online training process has to be run several times, resulting in a much higher cost compared to the ideal uplink payload $P$. 

Another benefit of FbFTL is the dataset balancing. If the overall dataset is imbalanced and hence the samples with certain types of output appears much more frequently than those of other outputs, it is hard for FL and FTL to distinguish this via gradient updates, and such imbalanced data distribution could significantly
degrade FL performance \cite{zhao2018federated, park2019wireless}. However, FbFTL with direct output information enables techniques, such as re-sampling specific classes or merging near-identical classes, to improve dataset imbalance.

One more benefit of FbFTL is to lift the constraint of $UC$ and $K_u$ on the training batch size. To avoid over-fitting to the training dataset, one needs to validate the performance on a separate validation set of data to identify the optimal number of training iterations to stop training and conclude the final model. It is straightforward for FbFTL to divide the obtained dataset into training and validation subsets. However, the gradient-based FL frameworks call for extra effort for the communication system to meticulously perform the training process and validation process with the desired order and number of samples. Additionally for the training process, due to the broadcast nature of the downlink in wireless FL and FTL, all selected clients in the same communication iteration receive the same parameters. Hence, each SGD mini-batch has a larger size than $\sum_{u=1}^{UC} K_u$, and an overwhelmingly large SGD mini-batch size may delay the training process and require more training iterations. The mini-batches in these gradient-based FL frameworks are also on the order of given clients' samples. Therefore, when the clients' sample distribution is biased, we cannot shuffle the samples between iterations and have to accept the loss in the final performance. However, FbFTL may choose any rational size of SGD mini-batch without the constraints of the communication system, and can reshuffle the data in each training iteration.

\subsection{Payload Analysis} \label{subsec:FbFTLpayload}
Note that FbFTL has the major benefit of requiring one-time communication between the clients and the PS. In addition to this, FbFTL has much less data in a single upload batch compared to that of each sample in the upload batch of FL and FTL.  Let us assume that in the fully connected layer $m$ with bias, the number of input nodes and the number of output nodes are denoted by $N^{-}_{m}$ and $N^{+}_{m}$, respectively Then, the number of trainable parameters in this layer is given by $T_m = N^{+}_{m} (N^{-}_{m} + 1)$. Compared to the dimension of $N^{-}_{m_c}$, the amount of information that we need to represent $\textbf{y}_{u,k} \in \{1,\ldots,  N\}$, i.e., $\log_2{N}$ bits, is negligible. Note that in FbFTL, gradients are computed at the PS. Therefore, FedAvg cannot be applied and each sample is required to be uploaded separately. Therefore for FbFTL, the uplink payload for each sample is $d N^{-}_{m_c}$, and the overall successful uplink payload required for training is 
\begin{equation} \label{eq:PayloadFbFTL}
P^{FbFTL} = d \sum_{u = 1}^{U}{K_u} N^{-}_{m_c} \  \text{bits}.
\end{equation}
Compared to the uplink payload in (\ref{eq:PayloadFTLclassify}) of FTL with FedAvg (in which only the task-specific sub-model is updated), the ratio of each upload for single sample between FTL and FbFTL is 
\begin{multline} \label{eq:FTL/FbFTL_single}
\begin{aligned} 
\frac{d \sum_{m=m_c}^{M}{T_m}}{d N^{-}_{m_c}} & = \frac{T_{m_c} + \sum_{m=m_c+1}^{M}{T_m}}{N^{-}_{m_c}} \\
& = \frac{N^{+}_{m_c} (N^{-}_{m_c} + 1) + \sum_{m=m_c+1}^{M}{T_m}}{N^{-}_{m_c}} \\
& > N^{+}_{m_c} .
\end{aligned}
\end{multline}
Therefore, 
\begin{multline} \label{eq:FTL/FbFTL_all}
\begin{aligned} 
\frac{P^{FTL}_{c}}{P^{FbFTL}} & = \frac{d I^{FTL}_{c} U C \sum_{m=m_c}^{M}{T_m}}{d \sum_{u=1}^{U}{K_u} N^{-}_{m_c}}  > \frac{I^{FTL}_{c} U C}{\sum_{u=1}^{U}{K_u}} N^{+}_{m_c}. \\
\end{aligned}
\end{multline}
The number of extracted features for many state-of-the-art models is larger than $10^3$, and TL usually requires a relatively deeper task-specific sub-model, and therefore $N^{+}_{m_c}$ can be large. For the cross-device federated learning, the clients do not obtain huge local datasets, $\sum_{u=1}^{U}{K_u} < I^{FTL}_{c} U C$. Therefore, we have $P^{FbFTL} \ll P^{FTL}_{c}$. We note that FbFTL and FTL that updates the task-specific sub-model have the same performance at every iteration because the only difference between the two methods is the communication model but not the numerical process to generate the gradient updates. Combining this observation with the conclusion in (\ref{eq:compareFLFTL}), we have 
\begin{equation} \label{eq:compareFLFTLFbFTL}
P^{FbFTL} \ll P^{FTL}_{c} < P^{FTL}_{f} < P^{FL} ,
\end{equation}
and therefore we expect extremely smaller uplink payload in FbFTL compared to FTL and FL.

We also note that FbFTL has the smallest downlink broadcast payload and the least local computation compared to FL and FTL. For FL, FTL that updates all parameters, FTL that updates the task-specific sub-model, and FbFTL, the overall downlink broadcast payloads, respectively, are
\begin{equation} \label{eq:broadcastFL}
D^{FL} = d I^{FL} \sum_{m=1}^{M}{T_m} \  \text{bits},
\end{equation}
\begin{equation} \label{eq:broadcastFTLf}
D^{FTL}_{f} = d I^{FTL}_{f} \sum_{m=1}^{M}{T_m} \  \text{bits},
\end{equation}
\begin{equation} \label{eq:broadcastFTLc}
D^{FTL}_{c} = d I^{FTL}_{c} \sum_{m=1}^{M}{T_m} \  \text{bits},
\end{equation}
\begin{equation} \label{eq:broadcastFbFTL}
D^{FbFTL} = d \sum_{m=1}^{m_c-1}{T_m} \  \text{bits}.
\end{equation}

\subsection{Time Complexity Analysis} \label{subsec:FbFTLcomplex}
We note that FbFTL consumes less computation time and power for training in total, and also transfers a proportion of computation from the client devices to the PS. In FL and FTL, each client must complete one full forward pass and one back-propagation of the parameters to be trained for each training sample at each iteration. However, in FbFTL, each sample is only processed once by the client, and each client only needs to complete the forward pass of feature extraction sub-model, while all other computations are completed at the PS. Such shift of computational load to the PS is particularly beneficial if the end users and devices are severely limited in their computational capabilities and power resources (for instance, in IoT networks) and the PS is equipped with advanced processors and has access to more resources. 


In this subsection, we use time complexity to estimate the computation time and energy consumption. For each fully connected layer $m_1$ with $N^{-}_{m_1}$ input nodes and $N^{+}_{m_1}$ output nodes, there are $T_{m_1} = N^{+}_{m_1} (N^{-}_{m_1} + 1)$ trainable parameters, and each sample requires $O(N^{+}_{m_1} (N^{-}_{m_1} + 1)) = O(T_{m_1})$ time complexity for matrix multiplication during forward pass, and the same time complexity for matrix multiplication during back-propagation. Each 2-dimensional convolutional layer $m_2$ with stride 1 in each direction and same padding has the input shape $\{a^-, b^-, c^-\}$ where $N^{-}_{m_2} = a^- b^- c^-$, kernel shape $\{c^+, a', b', c^-\}$ where $T_{m_2} = c^+ (a' b' c^- + 1)$, and output shape $\{a^-, b^-, c^+\}$ where $N^{+}_{m_2} = a^- b^- c^+$. Thus, each sample requires $O(c^+ (a' b' c^- + 1) a^- b^-) = O(a^- b^- T_{m_2})$ time complexity for matrix multiplication during forward pass, and the same during back-propagation. Comparatively, the time complexity of activation functions, max-pooling layers and dropout layers is negligible. Therefore, we denote the time complexity of each trainable layer $m$ for a single training sample as $O(X_m)$, and the overall time complexity required for all clients as $O(X)$. Specifically, the overall time complexities at clients for the methods of FL, FTL that updates all parameters, FTL that updates the task-specific sub-model, and FbFTL, respectively, are
\begin{equation} \label{eq:complexFL}
O(X^{FL}) 
\propto O\left( 2 I^{FL} U C \sum_{m=1}^{M}{X_m} \right) ,
\end{equation}
\begin{equation} \label{eq:complexFTLf}
O(X^{FTL}_{f}) 
\propto O\left( 2 I^{FTL}_{f} U C \sum_{m=1}^{M}{X_m} \right) ,
\end{equation}
\begin{equation} \label{eq:complexFTLc}
O(X^{FTL}_{c}) 
= O\left( U \sum_{m=1}^{m_c-1}{X_m} + 2 I^{FTL}_{c} U C \sum_{m=m_c}^{M}{X_m} \right) ,
\end{equation}
\begin{equation} \label{eq:complexFbFTLclient}
O(X^{FbFTL}) \propto O\left( U \sum_{m=1}^{m_c-1}{X_m} \right) ,
\end{equation}
where $\propto$ stands for ``proportional to". Typically, we have $X^{FbFTL} < X^{FTL}_{c} \ll X^{FTL}_{f} < X^{FL}$.

\subsection{Experimental Results}\label{sec:exp}

In this section, we consider the application of FL, FTL and FbFTL to the VGG-16 CNN model \cite{simonyan2014very} and transfer the knowledge learned from ImageNet dataset \cite{deng2009imagenet} to CIFAR-10 dataset \cite{krizhevsky2009learning}.

ImageNet is a vast online database containing over 14 million images, each with hand-annotated labels describing the  classification types or intended outputs for training. The pre-trained source model we use for TL was created on the 2012 ImageNet Large Scale Visual Recognition Challenge (ILSVRC2012, \cite{ILSVRC15}), and the images come from 1000 different categories. As an illustration, we provide 10 samples with their labels in Fig. \ref{fig:imagenet}.
\begin{figure}
	\centering
	\includegraphics[width=1\linewidth]{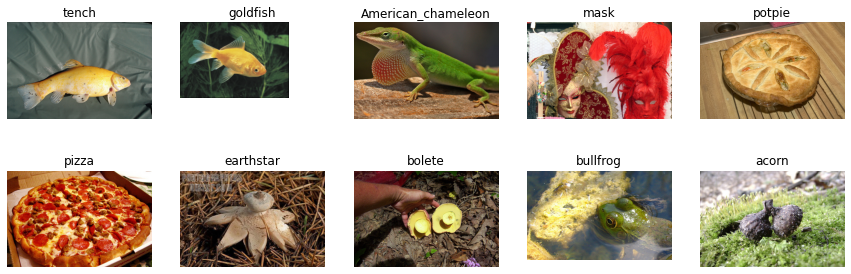}
	\caption{ImageNet samples with labels. }
	\label{fig:imagenet}
\end{figure}

CIFAR-10 is a database of $60000$ images, each with one of $N=10$ distinct labels. Out of these images, $50000$ images are used for training and $10000$  images for testing. Fig. \ref{fig:cifar10} showcases the first ten samples with their labels. Note that these samples have simple labels and appear more blurred compared to those in ImageNet due to their lower resolution/dimension.
\begin{figure}
	\centering
	\includegraphics[width=0.8\linewidth]{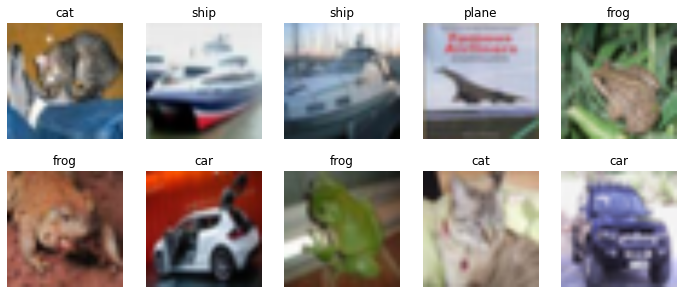}
	\caption{CIFAR-10 samples with labels. }
	\label{fig:cifar10}
\end{figure}

In Fig. \ref{fig:VGG16}, we depict the structure of VGG-16 used for training on CIFAR-10. For TL, we consider the first half of the layers marked blue as the feature extraction sub-model $f^1_{\pmb{\theta}^1}$ and directly transfer this sub-model from that trained on ImageNet. We deem the latter part marked yellow as the task-specific sub-model $f^2_{\pmb{\theta}^2}$ and randomly initiate this part. For FbFTL, the dimension of the intermediate output is $N^{-}_{m_c} = 4096$.
\begin{figure}
	\centering
	\includegraphics[width=0.6\linewidth]{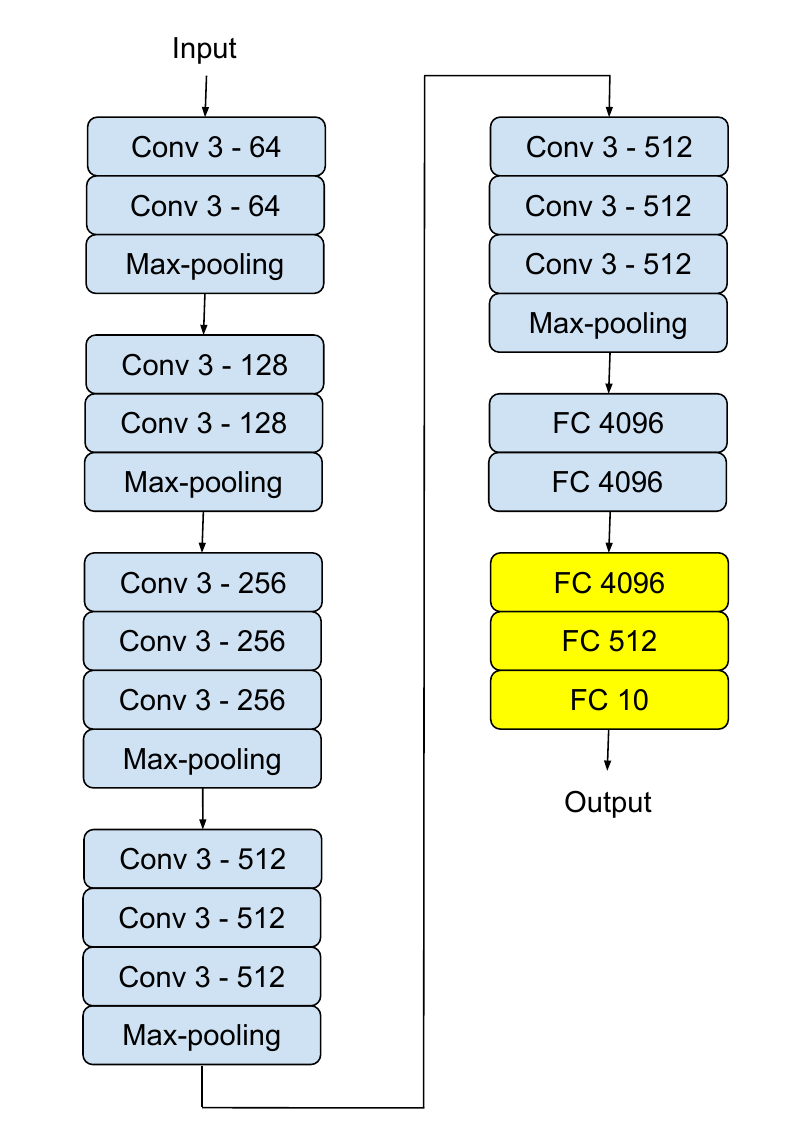}
	\caption{Diagram of the VGG-16 model for training on CIFAR-10 dataset. ``Conv 
 $\{$receptive field size$\}$ - $\{$number of output channels$\}$" depicts the convolutional layers. ``FC $\{$output size$\}$" depicts the fully connected layers. 
 }
	\label{fig:VGG16}
\end{figure}

To train the model on CIFAR-10, we utilize Nvidia GeForce GPU with CUDA to run the algorithms with PyTorch \cite{pytorch}. We assume that there are $U=6250$ clients in total, each iteration takes a fraction $C=1.28 \times 10^{-3}$ of all clients, and each batch contains $K_u=8$ samples. In the two most commonly used deep learning tools TensorFlow (including Keras) \cite{tensorflow2015-whitepaper} and PyTorch, the default data type of each number has 32 bits and therefore $d=32$ bits. The learning rate is $10^{-2}$, the momentum of the optimizer is $0.9$, and the L2 penalty is $5 \times 10^{-4}$.

\begin{table*}[t]
	\footnotesize
	\begin{center}
		\begin{tabular}{| l || l | l | l | l | l | l | l |}
			\hline
			 & FL & SFL & FL$^{low}$ & FTL$_f$ & FTL$_f^{low}$ & FTL$_c$ & FbFTL \\ \hline\hline
			number of upload batches & 656250 & 656250 & 68750 & 193750 & 25000 & 525000 & 50000  \\ \hline
			upload parameters per batch & 153144650 & 117483328 & 153144650 & 153144650 & 153144650 & 35665418 & 4096  \\ \hline
			uplink payload per batch & \textbf{4.9 Gb} & \textbf{3.8 Gb} & \textbf{4.9 Gb} & \textbf{4.9 Gb} & \textbf{4.9 Gb} & \textbf{1.1 Gb} & \textbf{131 Kb}  \\ \hline
			total uplink payload $P$ & \textbf{3216 Tb} & \textbf{2467 Tb} & \textbf{337 Tb} & \textbf{949 Tb} & \textbf{123 Tb} & \textbf{599 Tb} & \textbf{6.6 Gb}  \\ \hline
			total downlink payload $D$ & 402 Tb & 308 Tb & 42 Tb & 253 Tb & 15 Tb & 322 Tb & 3.8 Gb  \\ \hline
			computation time complexity $X$ & $1.63\times10^{17}$ & $1.63\times10^{17}$ & $1.7\times10^{16}$ & $4.80\times10^{16}$ & $6.19\times10^{15}$ & $1.07\times10^{15}$ & $7.74\times10^{14}$ \\ \hline
			validation accuracy & 89.42\% & 89.42\% & 86.64\% & 93.75\% & 86.45\% & 86.51\% & 86.51\%  \\ \hline
		\end{tabular}
		\caption{Performance comparison on VGG-16 between FL, SFL, FTL updating full model (FTL$_f$), FTL updating task-specific sub-model (FTL$_c$) and FbFTL. Additionally, we compare with cases in which FL and FTL$_f$ achieve the same accuracy as FTL$_c$ and FbFTL ($\approx 86\%$) by terminating training slightly earlier. Algorithms in these cases are referred to as FL$^{low}$ and FTL$_f^{low}$. }
		\label{tab:exp}
	\end{center}
\end{table*}

In Table \ref{tab:exp}, we compare the performances of FL, SFL, FTL updating the full model (FTL$_f$), FTL updating the task-specific sub-model (FTL$_c$), and FbFTL. For increased fairness in the comparison of payloads, we further demonstrate the performances  FL$^{low}$ and FTL$_f^{low}$, which are the FL and FTL$_f$ algorithms that terminate training process when the validation accuracy reaches those of FTL$_c$ and FbFTL (i.e., $\approx 86\%$). As we have analyzed, the other algorithms require $IUC$ successfully uploaded batches, while FbFTL only requires $\sum_{u=1}^{U}{K_u}$ batches. 
Also, FL, SFL, FTL$_f$ and FTL$_c$ require uploading $d \sum_{m=1}^{M}{T_m}$ bits, $d (K_u N^{-}_{m_c} + \sum_{m=m_c}^{M}{T_m})$ bits, $d \sum_{m=1}^{M}{T_m}$ bits, and $d \sum_{m=m_c}^{M}{T_m}$ bits, respectively, for each batch, while FbFTL only requires uploading $d N^{-}_{m_c}$ bits for each batch. 
In the third row of the table, we observe that the FbFTL algorithm significantly reduces the uplink payload per batch by four orders of magnitude (i.e. a factor of $10^{-4}$) compared to the other algorithms for each client. Additionally, in the fourth row, FbFTL leads to a reduction of five orders of magnitude (i.e. a factor of $10^{-5}$) in the total uplink payload during training when compared to the other algorithms.
These results demonstrate that FbFTL is apparently the most efficient scheme. FbFTL also results in a substantial decrease in the overall downlink payload. If larger models are trained for more complex tasks or if the size of the training dataset is more limited, the difference in payload could be even greater.
In (\ref{eq:complexFL}) through (\ref{eq:complexFbFTLclient}), we have described the overall computation time complexity required for training at clients as $O(X)$. In the second-to-last row of the table, we quantify and provide this time complexity for each method by counting the number of multiplications among floating numbers. We readily observe that FbFTL requires the least computation time and power consumption at the clients, and has two orders of magnitude reduction compared to FL, SFL and FTL$_f$, since it only runs the forward pass over feature extraction sub-model for each sample one time. We further note that for FbFTL, the computation complexity at the PS is also low with $X=3.00\times10^{14}$. Furthermore, we note that different number of clients $CU$ in each iteration does not affect the performance of FbFTL, but a large number of clients $CU$ reduces the performance of FL and FTL, since it defines the minimum ``training batch size" $CU K_u$. The performance of FL and FTL will decrease with an overwhelmingly large training batch size, which is another benefit of FbFTL. We in the experiment pick a small $CU=8$ so that the benchmark schemes (i.e., FL, FTL$_f$, and FTL$_c$) have the best performance (at which point  FTL$_c$ and FbFTL have the same performance). Moreover, we also observe that FTL$_c$ and FbFTL only update a small portion of the parameters and exhibit a slight decrease in validation accuracy, which is the trade-off for the reduced payload. However, we will show in Section \ref{sec:robust} and Section \ref{sec:privacy} that under communication efficiency or privacy constraints, FbFTL may also prevail in terms of validation accuracy in certain situations.

\begin{table*}[t]
    \small
	\begin{center}
		\begin{tabular}{ | l || l | l | l | l | l | l | l | }
			\hline
			 & FTL$_f$ & FTL$_c$ & FbFTL & FTL$_c$ & FbFTL & FTL$_c$ & FbFTL \\ \hline\hline
            number of trained encoders & 8 & 8 & 8 & 4 & 4 & 2 & 2    \\ \hline
			number of upload batches & 132588 & 36830 & 7366 & 88392 & 7366 & 103124 & 7366 \\ \hline
			upload parameters per batch & 109860224 & 60511616 & 1024 & 51070144 & 1024 & 46349504 & 1024  \\ \hline
			uplink payload per batch & \textbf{3.5 Gb} & \textbf{1.9 Gb} & \textbf{32.7 Kb} & \textbf{1.6 Gb} & \textbf{32.7 Kb} & \textbf{1.5 Gb} & \textbf{32.7 Kb}  \\ \hline
			total uplink payload $P$ & \textbf{466.1 Tb} & \textbf{71.3 Tb} & \textbf{241.4 Mb} & \textbf{144.5 Tb} & \textbf{241.4 Mb} & \textbf{152.9 Tb} & \textbf{241.4 Mb}  \\ \hline
			total downlink payload $D$ & 116.0 Tb & 32.2 Tb & 1.58 Gb & 77.3 Tb & 1.88 Gb & 90.2 Tb & 2.03 Gb \\ \hline
			validation ROUGE-1 & 45.9249 & 45.4680 & 45.4680 & 45.2827 & 45.2827 & 44.9862 & 44.9862 \\ \hline
		\end{tabular}
		\caption{Performance comparison on FLAN-T5-small between FTL$_f$ updating full model, FTL$_c$ updating task-specific sub-model and FbFTL, with different number of encoders trained.}
		\label{tab:exp_LLM}
	\end{center}
\end{table*}

Furthermore, we note that in transfer learning, there is an additional trade-off between privacy protection, performance and payload. When partitioning a model into feature extraction sub-model and task-specific sub-model, choosing the cut layer closer to the output better preserves data privacy, while picking a cut layer closer to the input improves the training performance. In order to demonstrate such a trade-off, we next consider a language model as our application scenario. In FL, FTL, and FbFTL on natural language processing tasks, we consider tasks where the model and the intermediate features are not proprietary or private, and thus we assume that the clients are willing to share them while keeping the local data private As an example, we show the results on a conversation summary task SAMSum \cite{gliwa-etal-2019-samsum} with 32128 distinct token types out of 14732 training dialogues and 819 testing dialogues. For instance, dialogue ID 13728867 in SAMSum is ``Olivia: Who are you voting for in this election? Oliver: Liberals as always. Olivia: Me too!! Oliver: Great", and ground truth summary is ``Olivia and Olivier are voting for liberals in this election." For this task, we deploy a language model FLAN-T5-small \cite{chung2022scaling}, which is a transformer with 110 million parameters, including 8 encoders and 8 decoders. This model is pre-trained and the dataset is relatively small, so we do not have FL in this case. We assume that there are $U = 7366$ clients, each has $K_u = 2$ dialogues, and each iteration takes a fraction $C = 5.43 \times 10^{-4}$ of all clients. Our experiment is based on HuggingFace \cite{wolf2020huggingfaces} with learning rate $2 \times 10^{-4}$. 

In Table \ref{tab:exp_LLM}, we compare the performances of FTL and FbFTL in terms of ROUGE-1 score, which measures the match between the generated text and reference text. A larger ROUGE-1 indicates better performance. We note that FTL$_f$ trains all components including encoders, decoders, embedding and the final linear layer. On the other hand, FTL$_c$ and FbFTL freeze embedding, and may or may not freeze several encoders close to the input prompts. The number of trained encoders is given in the second row of Table \ref{tab:exp_LLM}. For instance, the performance results in the third and fourth columns are for FTL$_c$ and FbFTL that have trained all encoders (and hence have not frozen any of them), while the other columns provide the performances with 4 or 2 encoders trained (indicating that 4 or 6 encoders close to the input are frozen). We do not need to freeze decoders since label privacy leakage converges to zero with shuffled batches, as will be shown in Section \ref{sec:privacy}. In Table  \ref{tab:exp_LLM}, we observe that FbFTL reduces the uplink and downlink payload by similar orders of magnitude as in the VGG-16 experiment. 
Furthermore, we notice that training less layers or encoders leads to a slight reduction in ROUGE-1 score but it does not guarantee lower payload. While FTL  may require more iterations to arrive convergence, FbFTL may need to broadcast a larger feature extraction sub-model.

\section{Robustness Analysis}\label{sec:robust}
In this section, we compare the performance of FbFTL with that of FL and FTL under the same packet loss rate (PLR) for each batch being uploaded, and illustrate the robustness of FbFTL against packet losses, data insufficiency, and compression, including quantization, sparsification and error feedback. 

\subsection{Packet Loss} \label{subsec:packet_loss}
As we have demonstrated in the previous section, gradient-based FedAvg FL frameworks including FL and FTL upload the gradient update $\textbf{g}_u = \sum_{k=1}^{K_u}{\nabla_{\pmb{\theta}}{L(f_{\pmb{\theta}} | \textbf{s}_{u,k})}}$ iteratively, where each batch contains the gradient summation from all samples of the client. In contrast, our proposed FbFTL uploads the data for each sample only once, and each batch contains extracted features $\textbf{z}_{u,k}$ and output $\textbf{y}_{u,k}$ from one sample. 

In (\ref{eq:FTL/FbFTL_single}), we have shown that each batch for gradient-based FL is much larger than each batch for FbFTL (by about $10^4$ larger in our experiments), and we assume that the packets to transmit both types of batches consist of multiple transmission blocks of the same size. For both types of packets, we consider measuring the robustness of all frameworks against packet loss caused by block losses due to network congestion or link outage (e.g., as a result of deep fading in wireless networks). For each learning framework whose packet consists of $n_b$ transmission blocks, we consider the same block loss rate (BLR). The PLR without retransmission is
\begin{equation} \label{eq:PLR_one_time}
\text{PLR} = 1 - (1 - \text{BLR})^{n_b}.
\end{equation}
Obviously, PLR highly depends on the value of $n_b$. FbFTL has significantly lower packet size and consequently we expect much lower PLR for the given same BLR. We  show the significant performance difference among different learning frameworks in our experiments at the end of this section.

If, on the other hand, we allow at most $n_r$ retransmissions for all packets and assume the channel state of each transmission to be independent and identically distributed (i.i.d.), we can lower the PLR and achieve similar packet-level reliability. However, this is realized at the cost of higher total uplink payload. Specifically for each learning framework requiring uplink payload $P$, 
the expected uplink payload with retransmissions is 
\begin{equation} \label{eq:PLR_retrans}
P(\text{BLR}) = \sum_{n'=1}^{n_b} \frac{P}{n_b} \sum_{n=0}^{n_r} (\text{BLR})^{n} = P \sum_{n=0}^{n_r} (\text{BLR})^n  .
\end{equation}
We note that $\sum_{n=1}^{n_r} (\text{BLR})^{n}$ is the expected number of retransmissions, which is the same for all different learning frameworks. If the transmission has a fixed bandwidth, $\sum_{n=1}^{n_r} (\text{BLR})^{n}$ also represents the delay factor of the total transmission. We further note that if $\text{BLR} = 0$, then $P(\text{BLR}) = P$. Otherwise, $P(\text{BLR}) = P +  P \sum_{n=1}^{n_r} (\text{BLR})^{n}> P$. Notice that the increase in the payload $P \sum_{n=1}^{n_r} (\text{BLR})^{n}$ grows with $P$ and $n_r$. Hence, learning frameworks with higher payload experience a higher increase in payload when retransmissions are introduced.  


\subsection{Data Insufficiency} \label{subsec:insuf_data}
At the end of Section \ref{subsec:FbFTLLearn}, we have discussed the benefit that FbFTL does not require additional online uploads from the clients to test different sets of hyper-parameters, while gradient-based frameworks need to run the entire uploading process multiple times. In practice, in addition to hyper-parameters, another intangible aspect prior to training is whether there exists a sufficient number of participating clients and training samples. If the planned set of samples is insufficient to train the neural network, gradient-based frameworks require rerunning the overall process with more clients, which leads to high consumption and potential waste of both computational and transmission resources. However, FbFTL only requires the new clients to compute and upload their batches, and hence there is no waste of transmission resources. 

\subsection{Quantization, Sparsification and Error Feedback} \label{subsec:quantization}
To further reduce the uplink payload of gradient-based frameworks, there have been extensive works on gradient compression, including gradient sparsification \cite{alistarh2018convergence, stich2018sparsified} and gradient quantization \cite{reisizadeh2020fedpaq, shlezinger2020federated}, especially signSGD, the extreme case in which each element is reduced to be binary valued without scaling \cite{bernstein2018signsgd}. Several recent studies utilize error feedback (or quantization and sparsification with memory) that reduces the error in compression at client's local device to improve the updates in future iterations \cite{richtarik2021ef21, basu2019qsparse}. 

For gradient-based frameworks utilizing gradient update $\textbf{g}_u$ with $X_g$ elements, we denote the sparsification function as $\mathcal{S}_r(\cdot): \mathbb{R}^{X_g} \rightarrow \mathbb{R}^{X_g}$, where $r \in (0,1]$. This function keeps the value of $r X_g$ elements in the vector with the highest absolute values, and set the value of all other elements to 0. We denote the quantization function as $\mathcal{Q}_q(\cdot): \mathbb{R}^{X_g} \rightarrow \mathbb{R}^{X_g}$ where $q \in \mathbb{Z}^+$, and this function quantizes each element in the vector to $q$ bits within the max/min range. 
Furthermore, we denote the error memory vector for client $u$ at iteration $t$ as $\textbf{m}_{u,t} \sim \mathbb{R}^{X_g}$. Therefore, the process of quantization and sparsification with error feedback at iteration $t$ is described as follows:
\begin{equation} \label{eq:errfdbk}
\textbf{g}'_u = \mathcal{Q}_q(\mathcal{S}_r(\textbf{g}_u + \textbf{m}_{u,t})) , 
\end{equation}
\begin{equation} \label{eq:errfdbk2}
\textbf{m}_{u,t+1} = \textbf{g}_u + \textbf{m}_{u,t} - \textbf{g}'_u , 
\end{equation}
where $\textbf{g}'_u$ is uploaded to the PS, and the local error memory vector is updated to $\textbf{m}_{u,t+1}$. For initialization, $\textbf{m}_{u,1} = \textbf{0}$. 

Similarly, for FbFTL utilizing extracted feature $\textbf{z}_{u,k}$ with $X_z$ elements, we may also apply sparsification and quantization. However, error feedback is not applicable, because there is only one upload iteration in FbFTL, and the extracted feature is not additive unlike the gradient. Therefore, the process of quantization and sparsification for FbFTL is described as follows:
\begin{equation} \label{eq:qsFbFTL}
\textbf{z}'_{u,k} = \mathcal{Q}_q(\mathcal{S}_r(\textbf{z}_{u,k})) , 
\end{equation}
where $\textbf{z}'_{u,k}$ is uploaded to the PS.

In both cases, the compression rate is close to $2^{d/q}/r$ where $d$ is the number of bits for the representation of the original data type. We note that there is also a potential drawback in the practical deployment of sparsification and error feedback. On the one hand, while all other operations including training, communication, inference, quantization and error feedback have time complexity no more than $O(X')$ where $X'$ is the total number of parameters in the neural network, we notice that depending on the sparsification ratio $r$, sparsification requires up to $O(X_g\log{X_g})$ time complexity for gradient-based frameworks and $O(X_z\log{X_z})$ time complexity for FbFTL. Typically, we have $O(X_z) \ll O(X_g) = O(X')$, and FbFTL can achieve more significant time complexity reduction in local sparsification computation compared to gradient-based frameworks. On the other hand, by compensating for the compression error in future iterations, the error feedback significantly improves the training performance when the compression rate is high, i.e., $r$ and $q$ are low. However, the error feedback requires $O(X^g)$ additional memory throughout the entire training process and not just during the local training. Also, we demonstrate below in the experimental results that the gradient-based frameworks do not prevail even with error feedback.


\subsection{Experimental Results} \label{subsec:exp_robust}
In this section, we numerically demonstrate the aforementioned robustness metrics in our experiments with the VGG-16 CNN model on transfer learning from ImageNet dataset to CIFAR-10 dataset.

In Fig. \ref{fig:PLR}, we show the PLR (without retransmissions) of different learning frameworks when the block size equals the batch size of FbFTL (i.e., the BLR equals the PLR of FbFTL). As we have shown in the row of \emph{uplink payload per batch} of Table \ref{tab:exp}, the batch sizes of other learning frameworks are about $10^4$ times larger than that of FbFTL. Due to the huge difference in the batch sizes, the PLR curves of the other learning frameworks almost reach $1$ when the PLR of FbFTL is less than $0.001$ (or equivalently when BLR$<0.001$). The difference is still substantially high if a few rounds of retransmissions are allowed in gradient-based frameworks. Fig. \ref{fig:PLR_part} provides the magnified plot of Fig. \ref{fig:PLR} for BLR$<0.001$, and we note that the curve of FL and the curve of FTL that updates full model overlap since they have the same number of parameters to update and therefore have the same batch size. In this figure, we again observe that the PLR curves of learning mechanisms other than FbFTL quickly approach $1$ within the considered range of BLR $<0.001$, while the PLR of FbFTL (being equal to BLR) stays below $0.001$. 

\begin{figure}[ht]
\centering
\begin{minipage}[t]{0.47\linewidth}
\includegraphics[width=1.0\linewidth]{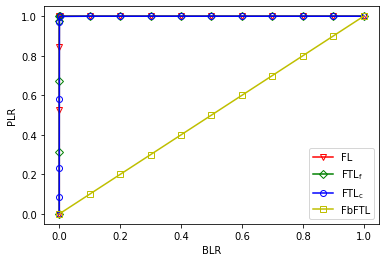}
\caption{Comparison of PLR (without retransmission) among FL, FTL updating full model, FTL updating task-specific sub-model, and FbFTL.}
\label{fig:PLR}
\end{minipage}
\quad
\begin{minipage}[t]{0.47\linewidth}
\includegraphics[width=1.0\linewidth]{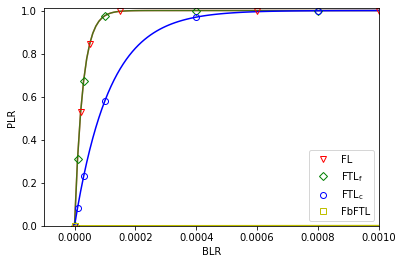}
\caption{Magnification of Fig. \ref{fig:PLR} for BLR$<$0.1\%.}
\label{fig:PLR_part}
\end{minipage}
\end{figure}

Since FbFTL has significantly lower PLR, we subsequently analyze its performance with different amount of data. In Fig. \ref{fig:insuf_data}, we see that FbFTL has relatively high validation accuracy even with only 10\% of the samples. Indeed, when we have only $0.1\%$ of the samples (i.e., the case with 50 samples), the accuracy is $82.5\%$ even with PLR$=0.5$. Furthermore, we observe that as PLR increases, the accuracy curves remain relatively flat, and experience sharp drops only when PLR approaches 1. Therefore, FbFTL is considerably robust against data insufficiency and PLR. We note that such robustness requires significantly more training iterations $I'$ such that $IU$ and $I'U'$ have the same order of magnitude, where $I$ is the original number of training iterations with all $U$ participating clients, and $I'$ is the number of training iterations with limited data from $U'$ clients. In the case of smaller number of participating clients $U' \ll U$ and PLR $=0$, FbFTL requires the same level of computational power consumption and more total downlink payload compared to the original case with $U$ clients. However, there is a huge reduction in total uplink payload and total local computational power consumption, because FbFTL only requires uploading once for each participating client. In comparison, gradient-based frameworks require uploading in every training iteration, and therefore there is no such benefit in terms of reduced uplink payload and local computations, while they also suffer from higher downlink payload. Similarly, we have the performance curves for FLAN-T5-small in Fig. \ref{fig:insuf_data_FLANT5}.
\begin{figure}
	\centering
	\includegraphics[width=0.7\linewidth]{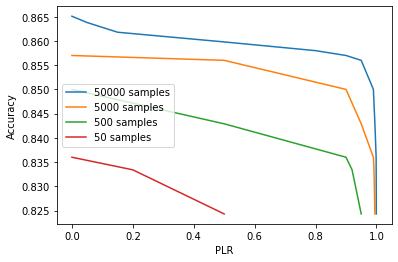}
	\caption{Validation accuracy of FbFTL with different number of participating samples and PLRs. }
	\label{fig:insuf_data}
\end{figure}
\begin{figure}
	\centering
	\includegraphics[width=0.7\linewidth]{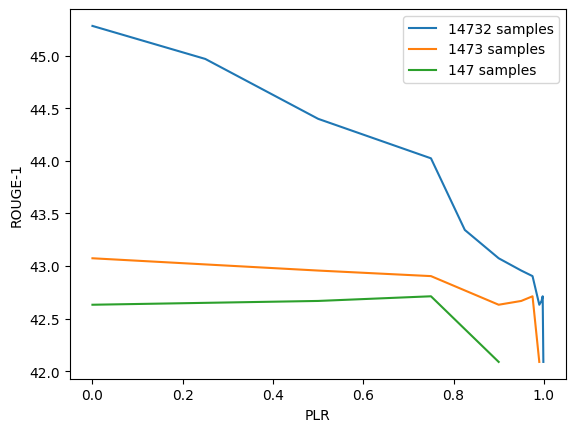}
	\caption{ROUGE-1 score of FbFTL with different number of participating samples and PLRs on FLAN-T5-small with SAMSum.}
	\label{fig:insuf_data_FLANT5}
\end{figure}

In Table \ref{tab:compress}, we show the validation accuracy of each framework with data from all $U$ clients and PLR $= 0$, but for different values of quantization size $q$ bits and sparsification ratio $r$. In the experiment, we pick the same set of values for $q$ and $r$ as in \cite{alistarh2018convergence, stich2018sparsified}. 
Note that the highest reduction in uplink payload is achieved when we set $r=0.001$ and $q=2$. However, even in this case, the uplink payload of gradient-based frameworks (i.e., FTL, FTL$_f$, FTL$_c$) is still greater than that of FbFTL with $r=1$ and $q=32$. Moreover, with such drastic reduction, gradient-based schemes achieve lower accuracies compared to that of FbFTL with $r=1$ and $q=32$. Therefore,  even without sparsification and more restrictive quantization (e.g., $q=8$ or $q=2$), FbFTL outperforms gradient-based frameworks in terms of both accuracy and payload reduction. We can further reduce the uplink payload of FbFTL by choosing $r < 1$ and $q < 32$. We note that FbFTL in the extreme case of $r=0.001$ only keeps 4 elements in the extracted features to distinguish among 10 classes. In this setting,  the information is severely limited and is insufficient to make an accurate prediction, resulting in accuracy levels of $65\%$ for FbFTL. If, on the other hand, we pick a certain threshold on the validation accuracy, for instance 82\%, FbFTL requires the sparsification ratio to be no lower than $r=0.01$ (which is $10$ times more than that of the best gradient-based framework), but still maintains a reduction of more than four orders of magnitude (i.e., reduction of $10^{-4}$) in total uplink payload compared to the gradient-based algorithms while achieving the same validation accuracy. Such good performance of FbFTL without error feedback is due to the robustness of extracted features against noise. If we consider the error from compression as random noise, the extracted features from a well-trained source model is typically robust to noise while gradient update does not necessarily have the same level of robustness. 

\begin{table}
	\begin{center}
		\begin{tabular}{ | p{10mm} || p{14mm} | p{14mm} | p{14mm} | p{15mm} | } \hline
		FL \newline FTL$_f$ \newline FTL$_c$ \newline FbFTL  & $r=1$  & $r=0.1$ & $r=0.01$ & $r=0.001$  \\ \hline\hline
		$q=32$  &  89.42\% \newline 93.75\% \newline 86.51\% \newline 86.51\%  &  83.96\% \newline 93.46\% \newline 85.45\% \newline 86.47\%  &  59.97\% \newline 90.98\% \newline 84.49\% \newline 82.10\%  &  42.26\% \newline 81.98\% \newline 82.78\% \newline 65.71\%  \\ \hline
		$q=8$  &  88.46\% \newline 93.36\% \newline 86.16\% \newline 86.41\%  &  82.34\% \newline 93.34\% \newline 85.44\% \newline 86.39\%  &  57.04\% \newline 90.83\% \newline 84.41\% \newline 81.85\%  &  39.34\% \newline 81.91\% \newline 82.53\% \newline 65.55\%  \\ \hline
		$q=2$  &  45.96\% \newline 88.20\% \newline 81.44\% \newline 85.50\%  &  44.67\% \newline 88.16\% \newline 81.29\% \newline 85.48\%  &  33.63\% \newline 86.19\% \newline 80.91\% \newline 80.99\%  &  31.91\% \newline 80.15\% \newline 78.61\% \newline 65.18\%  \\ \hline
		\end{tabular}
    	\caption{Validation accuracy comparison between FL, FTL updating full model, FTL updating task-specific sub-model and FbFTL with quantization to $q$ bits, sparsification ratio $r$ of elements, and error feedback (except for FbFTL).}
    	\label{tab:compress}
	\end{center}
\end{table}


\section{Privacy Analysis}\label{sec:privacy}
In previous sections, we have described the FbFTL framework where each client $u$ with $K_u$ samples uploads extracted features $\textbf{z}_{u,k}$ and output $\textbf{y}_{u,k}$ for $k = 1,\ldots, K_u$  instead of the gradient update $\textbf{g}_u = \sum_{k=1}^{K_u}{\nabla_{\pmb{\theta}}{L(f_{\pmb{\theta}} | \textbf{s}_{u,k})}}$ as done in FL and FTL with FedAvg. In this section, 
we conduct a privacy analysis by studying privacy leakage to a potential adversary, and propose protection strategies.  In particular, we consider leakage due to the unveiling of information regarding the outputs/labels $\{\textbf{y}_{u,k}\}$ (henceforth referred to as label privacy leakage) and the unveiling of intermediate features $\{\textbf{z}_{u,k}\}$ (henceforth referred to as feature privacy leakage). For example, during the training process of a classifier to distinguish the character in a photo to be a dog or a cat, the identity of the character being labeled as a dog is considered as label privacy, while the feature privacy includes additional information of the certain photo besides its label privacy such as the color of the fur and the furniture in the background. 

The label privacy leakage quantifies how much information about the batch of a targeted client is revealed to an adversary through the information on the outputs $\{\textbf{y}_{u,k}\}$ (or type of label in classification problems). The information on the outputs/labels can be of statistical nature or can be obtained via the unveiling of the outputs to the adversary. The former considers the information leakage via the knowledge of the general sample output distribution, while the latter specifies the leakage when the adversary has access to the output/label values. We analyze the label privacy leakage via the entropy and mutual information from the adversary's perspective and propose an uploading design that randomly shuffles all batches to conceal the dependency between the client address and the output data. 

Feature privacy leakage 
describes the amount of information that possibly leaks when the adversary obtains uploaded contents from clients. 
We  will analyze the feature privacy leakage via the differential privacy (DP) framework \cite{wei2020federated} and provide comparisons between FbFTL and the gradient-based frameworks (FL and FTL, in which cases feature privacy is leaked when the adversary obtains the gradient updates) through experiments and numerical results.

\subsection{Label Privacy Leakage} \label{subsec:label_pri}
First, we analyze the label privacy leakage. While FbFTL directly uploads the output in each batch for each sample, we note that FL and FTL with FedAvg that update the final layer also leak the output, and hence label privacy leakage also occurs in these cases. According to \cite{wainakh2022user}, the count of each output type in a batch can be numerically solved given the average gradient. Adversaries with certain prior knowledge on the training data are able to gain further knowledge and reconstruct the input from the average gradient, such as deep leakage \cite{zhu2019deep} or gradient inversion \cite{jeon2021gradient, huang2021evaluating}. Although the following label privacy analysis applies to both feature-based FbFTL and gradient-based FedAvg frameworks, we in label privacy analysis use the word ``batch" to indicate all uploaded information from one client with multiple samples. Compared to the high-dimensional gradient $\textbf{g}_u$ or feature $\textbf{z}_{u,k}$, the output $\textbf{y}$ also potentially reveals clients' private information but up to a certain degree. In this setting, we analyze the conditions under which the label privacy leakage (with statistical information) vanishes. We also address the role of shuffling the output information from all clients as a way of hiding the client's address from uploaded content when adversary has access to outputs.

To validate these approaches, we analyze the private output information leakage of a specific batch to a potential adversary without client addresses. According to \cite{longpr2017entropy}, the privacy loss of a query can be evaluated as the difference in the privacy before and after the query. In our setting, we determine the amount of label privacy via the uncertainty from the adversary's perspective, and quantify it by utilizing the entropy formulation. 

As shown in Fig. \ref{fig:batch}, we consider a  common learning task in which the output $\textbf{y} \in [0,1]^N$ is an axis-aligned unit vector with one element having a value of 1 indicating the label associated with the given input, and all others equal to 0. We assume that each sample $\{\textbf{x}_{u,k}, \textbf{z}_{u,k}, \textbf{y}_{u,k}\}$ is independent and identically distributed (i.i.d.). Each client $u \in \mathcal{U}$ transmits one batch with $K_u = K$ samples, including the outputs/labels $\{\textbf{y}_{u,k}\}_{k =1}^K$. Therefore for this client, there are $N^{K}$ different possible ordered batches of outputs in total (where $N$ is the number of possible different labels that can be associated with the input).
\begin{figure}
	\centering
	\includegraphics[width=0.9\linewidth]{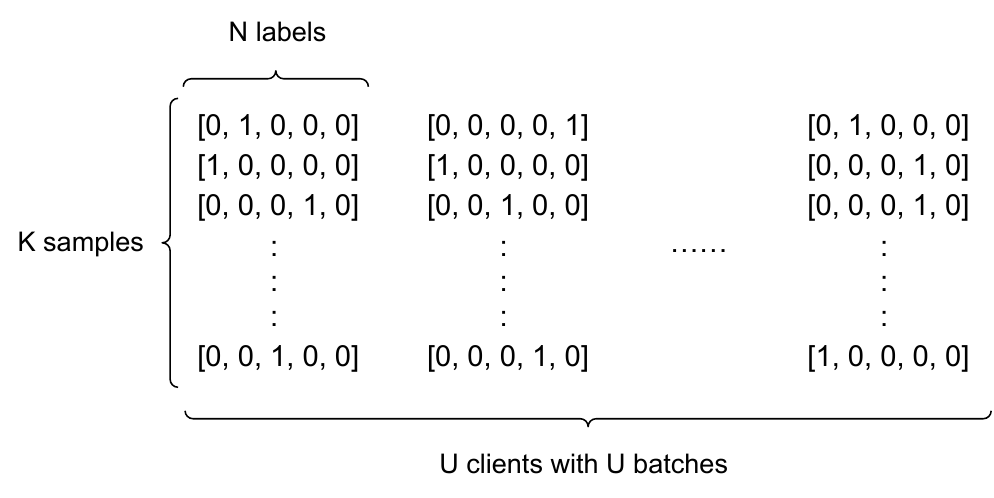}
	\caption{Diagram of every sample label $\textbf{y}_{u,k} \in [0,1]^N$ in batches. }
	\label{fig:batch}
\end{figure}

Assuming that the order in the batch does not contain information, we use $\mathcal{B}$ to denote the set of possible batches without order from client $u$, and the privacy information of the real batch $b$ from client $u$ is the sum vector $\sum^{K}_{k=1} \textbf{y}_{u,k} = [n^y_{1,b}, n^y_{2,b}, \ldots, n^y_{N,b}]$, where $\sum^{N}_{a=1} n^y_{a,b} = K$. Subsequently, the number of possible batches without order is the combination with the replacement of $N$ items taken $K$ times and is given by 
\begin{equation} \label{eq:size_B}
|\mathcal{B}| = \binom{N + K - 1}{K} .
\end{equation}
We denote the index of client $u$'s batch as a random variable $B$, and the index of the uploaded batch as $B = b \in \mathcal{B}$. We denote the type of the label $\textbf{y}_{u,k}$'s distribution as $P_y$. Typically to achieve better performance, in most of the machine learning applications we desire a uniform distribution over different labels, and we denote the uniform distribution over $N$ labels as $P_{y,\text{uni}}=[\frac{1}{N}, \frac{1}{N}, \ldots, \frac{1}{N}]$. 

To evaluate the adversary's knowledge on one given batch $B=b$ of a target victim client, we consider the presumed distribution $\textbf{P}_{B}$ of the batch $B$ from the adversary's perspective to quantify the adversary's uncertainty through the entropy $H(B)$ of the given batch.

\subsubsection{Label Privacy Leakage with Statistical Information} \label{subsubsec:mutual_label_pri}
A weak adversary without any prior knowledge of the sample distribution may presume the sample distribution to be uniform, i.e., $P_{y,\text{uni}}$. In this case, the distribution of ordered batches from the adversary's perspective is also uniform with probability $1/N^{K}$. Each batch index $b \in \mathcal{B}$ without order corresponds to $K! / {\prod^N_{a=1} n^y_{a,b}!}$ different batch indices with order, and hence the adversary's presumed probability distribution of batches without order (under the assumption of uniform sample distribution) is $\textbf{P}_{B|\text{uni}} = [p_{0,1}, p_{0,2}, \ldots, p_{0,|\mathcal{B}|}]$, where
\begin{equation} \label{eq:dist_without_order}
p_{0,b} = \frac{K!}{N^{K} \prod^N_{a=1} n^y_{a,b}!} . 
\end{equation}
Therefore from the weak adversary's perspective, the uncertainty of the batch from client $u$ with known format and size but without prior knowledge about the content can be quantified by the entropy
\begin{equation} \label{eq:entropy_u}
H(B\ |\ P_{y,\text{uni}} ) = - \sum^{|\mathcal{B}|}_{b=1} p_{0,b} \log_2{p_{0,b}} 
\end{equation}
where the condition in the entropy signifies that this is the entropy under the assumption of uniformly distributed samples. 

On the contrary, a stronger adversary with the ability to steal and decode a large amount of uploaded data may learn the structure of the DNN and is able to decode the labels of the $K$ samples in a batch, and hence such an adversary is likely to have the prior knowledge of the general distribution of batches from a large number of clients. In the case that the true sample distribution $P_{y,\text{true}}$ is known, we denote the strong adversary's presumed distribution of the batch from client $u$ as  
\begin{equation} \label{eq:dist_BgivenY}
\textbf{P}_{B|\text{true}} = [p_{1}, p_{2}, \ldots, p_{|\mathcal{B}|}] ,
\end{equation}
and the corresponding uncertainty at the adversary in the original setting before query is
\begin{equation} \label{eq:entropy_ugivenY_extreme}
H(B\ |\ P_{y,\text{true}}) = - \sum^{|\mathcal{B}|}_{b=1} p_{b} \log_2{p_{b}}   .
\end{equation}
By comparing the uncertainty between the weak adversary and the strong adversary, we give the definition of label privacy leakage with statistical information.

\begin{definition}\emph{Label Privacy Leakage with Statistical Information:}\label{def:mutual_privacy}
For an adversary without prior knowledge of the sample distribution (and hence that initially assumes uniform sample distribution), the label privacy leakage is the amount of information regarding the target batch $B$ that leaks to the adversary when it obtains the true distribution $P_{y,\text{true}}$. This privacy leakage can be formulated as follows: 
\begin{equation} \label{eq:mutual_privacy}
L^l_s = H(B \ |\ P_{y,\text{uni}}) - H(B \ | \ P_{y,\text{true}}) .
\end{equation}
\end{definition}

Note that $H(B \ |\ P_{y,\text{uni}})$ is the uncertainty in $B$ under the assumption that the labels are uniformly distributed. $H(B \ |\ P_{y,\text{true}})$ is the remaining uncertainty in $B$ when the true sample distribution is learned by the adversary. Hence, the difference is the information gained by or equivalently leaked to the adversary.


Machine learning tasks typically require dataset balancing. If the data collection process is sufficiently well designed so that each output type is almost equally likely and we have $\textbf{P}_{B|\text{true}} \to \textbf{P}_{B|\text{uni}}$, then the  label privacy leakage $L^l_s \to 0$. As we have illustrated in section \ref{subsec:FbFTLLearn}, achieving this goal is more viable in FbFTL.

\subsubsection{Label Privacy Leakage with Access via Query} \label{subsubsec:local_label_pri}
Next, we analyze the label privacy leakage when the adversary has access to the shuffled outputs (e.g., via a query).  In the worst case, the strongest query is the process that the adversary acquires all batches without clients' addresses. 
Specifically, we assume that the adversary has access to randomly shuffled $U$ batches. Recall that there are $|\mathcal{B}|$ different types of batches. We use $c(b)$ to denote the count for type $b$ batches among given $U$ batches. With this definition, we have $\sum_{b = 1}^{|\mathcal{B}|} c(b) = U$. For different set of $U$ batches, the counts will be different. We define $\mathcal{C}_U$ as a random vector of counts of different types of batches among a total of $U$ batches. Hence, for given $U$ batches, the realization of this vector is $\mathcal{C_U} = c_U = [c(1), c(2), \ldots, c(|\mathcal{B}|)]$. In the absence of any other statistical information, the adversary can utilize the following empirical distribution of $B$ based on the frequency of each batch type among all $U$ batches: 
\begin{equation} \label{eq:dist_BgivenS}
\textbf{P}_{B,\text{emp}} \triangleq \textbf{P}\{B\ |\ \mathcal{C_U} = c_U\} = \bigg[ \frac{c(1)}{U}, \frac{c(2)}{U}, \ldots, \frac{c(|\mathcal{B}|)}{U}\bigg] .
\end{equation}
With this empirical distribution based on the shuffled $U$ batches, the adversary's uncertainty is given by the entropy
\begin{equation} \label{eq:entropy_shuffled}
H(B\ |\ \mathcal{C_U} = c_U) = - \sum^{|\mathcal{B}|}_{b=1} \frac{c(b)}{U} \log_2{\frac{c(b)}{U}} .
\end{equation}

In the case in which the adversary has acquired $U$ batches via the query and knows the distribution $P_{y,\text{true}}$ of the samples, we denote the adversary's presumed distribution as $\textbf{P}(B\ |\ \mathcal{C_U} = c_U,  P_{y,\text{true}})$, and its uncertainty on client $u$'s batch as $H(B\ |\ \mathcal{C_U} = c_U, P_{y,\text{true}})$. 

First, we establish the following result.

\begin{lemma}
   Assume that the true sample distribution is known. Once $U$ batches are revealed to the adversary, the distribution of the batches (from the adversary's perspective) depends only on the frequency/count of each batch type in the unveiled $U$ batches and not on the sample distribution, i.e., 
   \begin{align}
       P(B=b\ |\ \mathcal{C_U} = c_U, P_{y,\text{true}}) = P\{B=b\,|\,\mathcal{C_U} = c_U\}.
   \end{align}
\end{lemma}

\begin{proof}
We prove this lemma by utilizing the Bayes' rule and determining the ratio of two conditional probabilities as follows:
\begin{multline} \label{eq:ys_equal_s}
\begin{aligned} 
P(B=b\ |\ \mathcal{C_U} =  c_U, P_{y,\text{true}}) =& \frac{P(\mathcal{C_U} = c_U,B=b|P_{y,\text{true}}) }{P(\mathcal{C_U} = c_U|P_{y,\text{true}})} \\
= & \frac{\frac{(U-1)!}{(c(b)-1)! \prod_{b'\neq b} c(b')!}}{\frac{U!}{c(b)! \prod_{b'\neq b} c(b')!}}
\\
=&\frac{\frac{(U-1)!}{(c(b)-1)! }}{\frac{U!}{c(b)! }}\\ 
=& \frac{c(b)}{U} 
\\
=&  P\{B=b\,|\,\mathcal{C_U} = c_U\}.
\end{aligned}
\end{multline}    
\end{proof}

This result indicates that the batch distribution is equal to the empirical distribution in (\ref{eq:dist_BgivenS}). Consequently, we also have the following characterization for the entropies:
\begin{align}
H(B\ |\ \mathcal{C_U} = c_U, P_{y,\text{true}}) = H(B\ |\ \mathcal{C_U} = c_U)
\end{align}
Next, we give the definition of the label privacy leakage with query, and identify a condition under which the privacy leakage vanishes.

\begin{definition}\emph{Label Privacy Leakage with Query:}\label{def:local_privacy}
For an adversary with prior knowledge of the true sample distribution $P_{y,\text{true}}$, the label privacy leakage with query is the amount of information regarding the batch of a target client that leaks to the adversary when it obtains the randomly shuffled set of uploaded $U$ batches including the target's batch. This privacy leakage can be formulated as follows: 
\begin{multline} \label{eq:local_privacy}
\begin{aligned}
L^l_q & = H(B\ |\, P_{y,\text{true}}) - H(B\ |\ P_{y,\text{true}}, \,  \mathcal{C_U} = c_U) \\
& = H(B\ |\, P_{y,\text{true}}) - H(B\ |\   \mathcal{C_U} = c_U) .
\end{aligned}
\end{multline}
\end{definition}


\begin{lemma}\label{lem:local_privacy} 
As the number of shuffled batches goes to infinity, the label privacy leakage with query converges to 0, i.e., 
\begin{equation} \label{eq:leakage_shuffled_zero}
\lim\limits_{U\to\infty} L^l_q = 0 . 
\end{equation}
\end{lemma}

\begin{proof}
As $U$ grows without bound, we have $\lim\limits_{U\to\infty} \textbf{P}_{B,\text{emp}} = \textbf{P}_{B|\text{true}}$ by the law of large numbers, and as a result, we have the characterization that $\lim\limits_{U\to\infty} H(B\ |\ \mathcal{C_U} = c_U) = H(B\ |\ P_{y,\text{true}})$. Hence, the  label privacy leakage $L^q_l$ with query converges to 0. 
\end{proof}


Note that if the adversary does not initially know even the distribution of the samples and assume a uniformly distributed samples, then total label privacy leakage after the unveiling of the $U$ batches to the adversary can be defined as 
\begin{align}
L_t^l = L_s^l + L_q^l = H(B \ |\ P_{y,\text{uni}}) - H(B \ | \ \mathcal{C_U} = c_U ).
\end{align}

If the shuffled dataset is perfectly balanced and we have $\textbf{P}_{B,\text{emp}} = \textbf{P}_{B|\text{uni}}$, the total label privacy leakage is zero, i.e., $L_t^l = L^l_s + L^l_q = 0$.

\subsection{Feature Privacy Leakage} \label{subsec:inst_pri}
In this section, we analyze the privacy leakage on the input $\textbf{x}$ via uploaded content. Such privacy leakage includes the information that is not necessarily needed to determine the output $\textbf{y}$, and we will define it as feature privacy. Different from label privacy where the adversary knows the implication of each type of output (or the meaning of each label), it is hard to define the adversary's prior knowledge on the gradients and extracted features, and it is unlikely to cancel feature privacy leakage via shuffling. Furthermore, it is hard to quantify and analyze the entropy and mutual information of gradients and extracted features because of the complexity of neural networks. Instead, we will analyze the feature privacy leakage via differential privacy.

For FbFTL, the intermediate output $\textbf{z}$ is also referred to as the smashed data in split learning \cite{gupta2018distributed, vepakomma2018split}, and cannot be directly transformed back to the input $\textbf{x}$ due to the nonlinearity of the activation functions in each layer. However, it is possible that FbFTL leaks privacy to some extent and partially reveals the input. Since the feature before a fully connected layer from a sample can be analytically solved from its gradient \cite{geiping2020inverting}, it is possible that gradient-based frameworks also leak the input. The strategy to extract the input from FedAvg gradients includes deep leakage \cite{zhu2019deep, geiping2020inverting, zhao2020idlg} and gradient inversion \cite{jeon2021gradient, huang2021evaluating}, and is extensively studied for image recognition. However, deep leakage highly depends on the dataset and DNN structure. Therefore, an analysis of the privacy leakage beyond label privacy is needed to identify what may be partially revealed regarding the input $\textbf{x}$ from the uploaded features.

Differential privacy (DP) provides an upper bound on the privacy leakage using a different measuring approach, 
and we compare the feature privacy preserving performances of FbFTL and other schemes via DP. In the DP analysis of feature privacy, we denote the training samples for each batch as $\textbf{d}$ and the function that generates the upload vector as $g$. More specifically, $\textbf{d} = \textbf{s}_{u,k}$ and $g(\textbf{d}) = \textbf{z}_{u, k} = f^1_{\pmb{\theta}^1}(\textbf{x}_{u,k})$ for FbFTL, while $\textbf{d}=\bigcup_{k = 1}^K \textbf{s}_{u,k}$ and $g(\textbf{d}) = \textbf{g}_u = \sum_{k=1}^{K}{\nabla_{\pmb{\theta}}{L(f_{\pmb{\theta}} | \textbf{s}_{u,k})}}$ for gradient-based FedAvg frameworks. Similarly as in \cite{dwork2006calibrating, dwork2011firm, dwork2014algorithmic}, we define the feature privacy through the following condition: $g$ satisfies $(\epsilon, \delta)$-DP if for any subset of possible outputs $G$ it holds that 
\begin{equation} \label{eq:ints_privacy_DP}
\text{Pr} \{ g(\textbf{d}) \in G \} \le e^{\epsilon} \text{Pr} \{ g(\textbf{d}') \in G \} + \delta ,
\end{equation}
where $\textbf{d}$ and $\textbf{d}'$ differ in a single sample with the true distribution $P_{y, \text{true}}$.

In \cite{abadi2016deep}, it has been shown via moments accountant approach that adding Gaussian noise $\textbf{n}$ to $g(\textbf{d})$ prior to transmission maintains an overall privacy loss of $(\epsilon, \delta)$. Typically, the variance of the noise depends on the maximum distance $\max{ \| g(\textbf{d}) - g(\textbf{d}') \| }$ for any two adjacent inputs $\textbf{d}$ and $\textbf{d}'$. To provide a fair comparison between different learning frameworks, we consider the maximum relative distance
$R = \max\limits_{d, d'}{\frac{ \max \left| g(\textbf{d}) - g(\textbf{d}') \right| }{ K \text{std} (  g(\textbf{d}) ) }}$ within training set for each framework, where $| \cdot |$ denotes absolute value and $\text{std}( \cdot )$ denotes the standard deviation. 
Thus, the additive Gaussian noise to each output $g(\textbf{d}_1)$ should be $\textbf{n} \sim \mathcal{N}(0,\,\sigma^2 R^2 \text{std}^2 (  g(\textbf{d}_1) ) \textbf{I})$ to mitigate the feature privacy leakage, and the lower bound of the $\sigma$ value can be determined via the moments accountant approach. 

As a result, according to \cite[Theorem 1]{abadi2016deep}, there exist constants $c_1$ and $c_2$ such that for any $\epsilon < c_1 C^2 I$, the training process with $g(\textbf{d})+\textbf{n}$ is $(\epsilon, \delta)$-differentially private for any $\delta>0$ if we choose $\sigma \ge \frac{ c_2 C}{\epsilon}\sqrt{ - I \log{\delta}} $,
where $I$ is the number of iterations ($I=1$ for FbFTL), and $C$ is the fraction of clients selected in each iteration. When the value of $\delta$ and noise factor $\sigma R$ are fixed, $\epsilon$ for FedAvg DP decreases as the batch size $K$ grows, but increases as training iterations $I$ (and potentially, the number of retraining for hyper-parameter tuning) increase. However, $\epsilon$ for FbFTL DP does not depend on these factors. Furthermore, the final performance also depends on the robustness of each framework against noise. Therefore, we demonstrate the comparisons via experiments in the following subsection.

\subsection{Experimental Results} \label{subsec:exp_pri}

In our experiments, we utilize the dataset CIFAR-10 which has $N=10$ types of labels, and is well-balanced (i.e., $P_{y,\text{true}}=P_{y, \text{uni}}$). We generate shuffled batches according to the uniform sample distribution $P_{y, \text{uni}}$ and evaluate the total label privacy leakage in different scenarios in Fig. \ref{fig:label_leakage}. The blue curve plots the initial uncertainty $H(B\ |\ P_{y, \text{uni}})$ for  different values of $K$ (where $K$ is the number of samples from each client). Each of the other curves shows the uncertainty given shuffled data $H(B\ |\ \mathcal{C}_U = c_U)$ with fixed total amount of samples $UK$ from all clients, and the number of clients $U$ is determined by each value of $K$ accordingly. By Definition~\ref{def:mutual_privacy} and Definition~\ref{def:local_privacy}, the total label privacy leakage given shuffled data is $L_t^l = L^l_s + L^l_q$ and is equal to the difference $H(B\ |\ P_{y, \text{uni}}) - H(B\ |\ \mathcal{C}_U = c_U)$. We see that $H(B\ |\ P_{y, \text{uni}})\approx H(B\ |\ \mathcal{C}_U = c_U)$ when $K$ is small, and thus the label privacy is well preserved. However, for given value of the product $UK$, $H(B\ |\ \mathcal{C}_U = c_U)$ diminishes fast once a certain threshold of $K$ is exceeded. Therefore, we note that the label privacy leakage becomes high when $K$ is large. As addressed before, the label privacy applies in the same way to both FbFTL and gradient-based FedAvg frameworks, and hence a smaller $K$ might be preferred in order to preserve the label privacy. However, current FL privacy analyses typically focus on DP that provides a relatively loose upper bound on the total local privacy. Within the DP setting, studies typically  consider a small number of clients and assume that  each client obtains hundreds of samples, if not thousands of, to achieve improved DP performance. 
However, as we have observed above, a setting with higher values of $K$  can lead to higher label privacy loss (in addition to requiring more training epochs and achieving less validation accuracy). This leads to the important conclusion that the balance between different types of privacy leakage needs to be considered carefully.
\begin{figure}
	\centering
	\includegraphics[width=1.\linewidth]{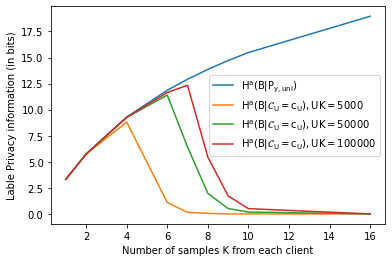}
	\caption{Label privacy of each client in bits against the number of samples $K$ from each client. }
	\label{fig:label_leakage}
\end{figure}

To compare the DP performances of feature privacy leakage, we apply different noise levels $\sigma$ on upload content $g(\textbf{d})$ of different frameworks to achieve the same level $(\epsilon, \delta)$ of privacy, and compare the validation accuracy at convergence with different $K$ in Fig. \ref{fig:DP_all}. In our experiments, we implement the moments accountant via Rényi Divergence-based DP accountant to get a tighter bound. With fixed $\delta=10^{-6}$, each plot shows the validation accuracy at convergence as a function of $\epsilon$ for given $K$. We note that higher $\epsilon$ indicates weaker DP protection, lower noise level $\sigma$, and thus higher accuracy (closer to the original performance without noise). As shown in Fig. \ref{fig:label_leakage}, label privacy is better preserved at $K=4$, and we show the DP performance within the same setting in Fig. \ref{fig:DP_K4}. While gradient-based FedAvg frameworks totally fail, FbFTL has good performance because it has much smaller data size $|g(\textbf{d})|$ and is more robust against noise, as it receives constant data instead of varying perturbation in each training iteration and converges to a sub-optimal model. Such advantage of FbFTL remains until $K$ increases to 100 as shown in Fig. \ref{fig:DP_K100}, where FTL that updates the task-specific sub-model outperforms FbFTL at $\epsilon \ge 4$ as they have the same data size $|g(\textbf{d})|$ while the noise level of FedAvg decreases as $K$ increases. However, the label privacy completely leaks for $K \ge 16$ as we 
have seen in Fig. \ref{fig:label_leakage}. FedAvg for $K \ge 16$ only prevails in protecting the subset of feature privacy that is not mitigated via shuffling. In Fig. \ref{fig:DP_K600}, we show the comparison for even larger $K=600$, where all types of FedAvg perform better than FbFTL.

\begin{figure}[ht]
\centering
\begin{subfigure}[t]{0.47\linewidth}
\includegraphics[width=1.0\linewidth]{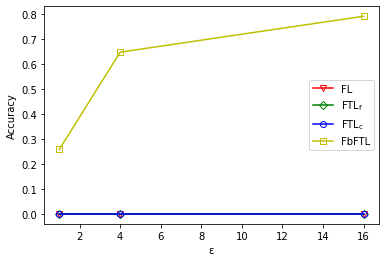}
\caption{K=4.}
\label{fig:DP_K4}
\end{subfigure}
\quad
\begin{subfigure}[t]{0.47\linewidth}
\includegraphics[width=1.0\linewidth]{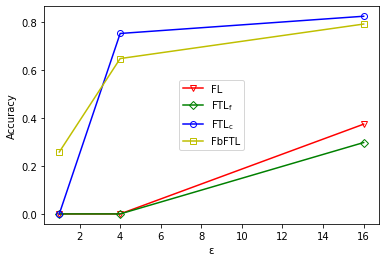}
\caption{K=100.}
\label{fig:DP_K100}
\end{subfigure}
\quad
\begin{subfigure}[t]{0.47\linewidth}
\includegraphics[width=1.0\linewidth]{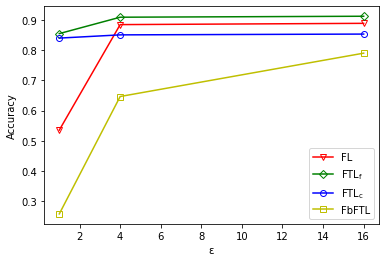}
\caption{K=600.}
\label{fig:DP_K600}
\end{subfigure}
\caption{Comparison of validation accuracy at different $K$, $\epsilon$, and $\delta=10^{-6}$ between FL, FTL updating full model, FTL updating task-specific sub-model and FbFTL. The accuracy is marked as 0 if the model fails to converge.}
\label{fig:DP_all}
\end{figure}

Note that it is extremely difficult to analytically evaluate the volume of feature privacy leakage that can be mitigated via shuffling. Howeover, we can compare the performances of FbFTL and FedAvg for a given $K$ via experimental results as done above. One key conclusion we have is the following. In terms of privacy preservation, FbFTL is preferred when $K$ is small (i.e., each client obtains a small set of samples), and gradient-based FedAvg is preferred when $K$ is large.






\section{Conclusion}\label{sec:con}
In this paper, we have presented a novel communication-efficient federated transfer learning method. In this proposed feature-based federated transfer learning (FbFTL), the features and outputs are uploaded rather than the gradient updates. We have provided a thorough description of the system design and the learning algorithm, and compared its theoretical payload with that of federated learning and federated transfer learning. Our results demonstrate substantial reductions in both uplink and downlink payload when using FbFTL. Via experiments, we have further shown the effectiveness of the proposed FbFTL by showing that FbFTL reduces the uplink payload by up to five orders of magnitude compared to that of existing methods. 
Subsequently, we have demonstrated that FbFTL with small batch size has significantly less packet loss rate than gradient-based frameworks, and illustrated its robustness against data insufficiency and quantization. 
Finally, we have considered different types of privacy leakage, and analyzed mitigation approaches. Specifically, we have first analyzed label privacy leakage with both statistical knowledge and query (resulting in access to the label outputs). We have identified a condition under which label privacy vanishes. We have also addressed feature privacy and considered a DP mechanism for preserving this privacy. We have shown that with small $K$ and shuffling that eliminates  label privacy leakage, FbFTL also attains good differential feature privacy protection. These characterizations and results render FbFTL a communication-efficient, robust, and privacy-preserving novel federated transfer learning scheme.

\bibliographystyle{IEEEtran}
\bibliography{ref.bib}

\end{document}